\documentclass{article}[preprint]

\usepackage{microtype}
\usepackage{graphicx}
\usepackage{subfigure}
\usepackage{booktabs} 
\usepackage{xspace}
\usepackage[preprint]{neurips_2024}

\usepackage{hyperref}

\newcommand{\theHalgorithm}{\arabic{algorithm}}

\usepackage{amsmath}
\usepackage{amssymb}
\usepackage{mathtools}
\usepackage{amsthm}

\usepackage{multicol}
\usepackage[capitalize,noabbrev]{cleveref}

\usepackage{amsmath, amsthm, amssymb}
\usepackage{bm}
\usepackage{graphicx}
\usepackage{xcolor}
\usepackage{url}
\usepackage{multirow}

\usepackage{algorithm}
\usepackage{algorithmic}
\usepackage{lipsum}

\usepackage[capitalize,noabbrev]{cleveref}

\theoremstyle{plain}
\newtheorem{theorem}{Theorem}[section]
\newtheorem{proposition}[theorem]{Proposition}
\newtheorem{lemma}[theorem]{Lemma}

\theoremstyle{definition}
\newtheorem{definition}[theorem]{Definition}

\theoremstyle{remark}

\usepackage[textsize=tiny]{todonotes}

\newcommand{\pure}{{\tt pure}}
\newcommand{\Ex}{{\cal E}}
\newcommand{\desc}{{\tt desc}}
\newcommand{\leaves}{{\tt leaves}}
\newcommand{\dist}{{\tt dist}}

\newcommand{\avg}{{\tt avg}}
\newcommand{\radius}{{\tt radius}}

\newcommand{\avgdist}{{\tt AvgDist}}
\newcommand{\diamset}{{\tt diam}}
\newcommand{\maxdiamset}{{\tt max\mbox{-}diam}}
\newcommand{\avgdiamset}{{\tt avg\mbox{-}diam}}
\newcommand{\complink}{{\tt complete-linkage}}
\newcommand{\singlelink}{{\tt single-linkage}}
\newcommand{\avglink}{{\tt average-linkage}}
\newcommand{\minimax}{{\tt minimax}}
\newcommand{\cost}{{\tt cost}}

\newcommand{\link}{{\tt Link}}
\newcommand{\Pts}{{\tt Pts}}
\newcommand{\phisum}{\phi_{\Sigma}}

\newcommand{\avec}{\mathbf{a}}

\newcommand{\remove}[1]{}

\newcommand{\red}[1]{{\color{red} #1}}

\newcommand{\OPTAVG}{\textrm{OPT$_{{\tt AV}}$}\xspace}
\newcommand{\OPTDIM}{\textrm{OPT$_{{\tt DM}}$}\xspace}
\newcommand{\OPTDIMK}{\textrm{OPT$_{{\tt DM}}(k)$}\xspace}
\newcommand{\OPTAVGK}{\textrm{OPT$_{{\tt AV}}(k)$}\xspace}
\newcommand{\A}{\mathcal{A}}
\newcommand{\mom}[1]{{\left\vert\kern-0.25ex\left\vert\kern-0.25ex\left\vert #1 \right\vert\kern-0.25ex\right\vert\kern-0.25ex\right\vert}}

\renewcommand{\L}{\mathcal{L}}

\newcommand{\C}{\mathcal{C}}
\newcommand{\T}{\mathcal{T}}
\newcommand{\Diam}{{\tt DM}}

\title{New bounds on the cohesion of complete-link and other linkage methods for agglomeration clustering}

\author{
  Sanjoy Dasgupta \\ sadasgupta@ucsd.edu \\  University of California San Diego  \\
\And
  Eduardo S. Laber \\ laber@inf.puc-rio.br \\ PUC-RIO \\
  }

\begin{document}
\maketitle

\begin{abstract}
\remove{ Linkage methods such as \complink \,  are
among the most popular algorithms for hierarchical clustering.
Despite their relevance the current knowledge regarding the quality
of the clustering produced by these methods is quite limited.
Here, we improve the currently available
bounds on the maximum diameter of the clustering obtained by
\complink \, for metric spaces.

One of our new bounds, in contrast to the existing ones, allows us to 
separate \complink \, from \singlelink \, in terms
of approximation w.r.t. to the diameter,
aligning with the common perception 
that the former is more suitable than the latter
when the goal is producing compact clusters.
}

 Linkage methods are
among the most popular algorithms for hierarchical clustering.
Despite their relevance the current knowledge regarding the quality
of the clustering produced by these methods is limited.
Here, we improve the currently available
bounds on the maximum diameter of the clustering obtained by
\complink \, for metric spaces.

One of our new bounds, in contrast to the existing ones, allows us to 
separate \complink \, from \singlelink \, in terms
of approximation for the diameter,
which corroborates the common perception 
that the former is more suitable than the latter
when the goal is producing compact clusters.

We also show that our techniques 
can be employed to derive upper bounds on the cohesion of a class of linkage methods that includes the quite popular \avglink. 

\end{abstract}

\remove{
Moreover, 
our new bounds allow us to 
separate \complink \, from \singlelink \, in terms
of approximation w.r.t. to the diameter,
aligning with the common perception 
that the former is more suitable than the latter
when the goal is producing compact clusters.
}

\onecolumn

\section{Introduction}

Clustering is the problem of partitioning a set of items so that similar items are grouped together and dissimilar items are separated. It is a fundamental tool in machine learning that is commonly used for exploratory analysis and for reducing the computational resources required to handle large datasets. For comprehensive descriptions of different clustering methods and their applications, we refer to \citep{Jain:1999,HMMR15}. 

One important type of clustering is hierarchical clustering.
Given a set  of $n$ points, a hierarchical clustering is a sequence of
clusterings $({\cal C}^0,{\cal C}^{1},\ldots,\C^{n-1})$,
where ${\cal C}^0$ is a clustering with $n$ unitary clusters, 
each of them corresponding to one of the $n$ points,
and  ${\cal C}^{i}$, for $i \ge 1$, is obtained from ${\cal C}^{i-1}$ 
 by replacing two clusters of ${\cal C}^{i-1}$ with their union.
Hierarchical clustering algorithms are implemented in widely used machine learning libraries such as {\tt scipy} and they have applications in many contexts such as in the study of evolution through phylogenetic trees \citep{b684a7bdefca4810bdaaf10bf8196c46}. 

There is a significant literature on hierarchical clustering; for  good surveys we refer to
\citep{10.1093/comjnl/26.4.354,DBLP:journals/widm/MurtaghC12}.
With regards to more theoretical work, one important line of research consists of designing algorithms for hierarchical clustering
 with provable guarantees for natural optimization criteria such as cluster diameter and the sum of quadratic errors \citep{DASGUPTA2005555,DBLP:journals/siamcomp/CharikarCFM04,DBLP:journals/siamcomp/LinNRW10,DBLP:conf/esa/ArutyunovaR22}.
Another relevant line aims
to understand the theoretical properties (e.g. approximation guarantees)
of algorithms widely used in practice, such as linkage methods
\citep{DASGUPTA2005555,DBLP:journals/corr/abs-1012-3697,DBLP:conf/esa/GrosswendtR15,arutyunova_et_al:LIPIcs.APPROX/RANDOM.2021.18,DBLP:conf/soda/GrosswendtRS19}.

Here, we contribute to this second line of research by
 giving new and improved analysis for the
  \complink  \,\citep{DBLP:journals/jmlr/AckermanB16}
 and also for a class of linkage methods
 that includes 
\avglink\, \citep{DBLP:journals/jmlr/AckermanB16} and \minimax \, \citep{BienTib2011}.




 
\subsection{Our Results}
Let  $({\cal X},dist)$ be a metric space, where ${\cal X}$ is a set of $n$ points.
The diameter $\diamset(S)$ of a set of points $S$ is  given by
$\diamset(S)=\max\{dist(x,y)|x,y \in S\}$.
A \emph{$k$-clustering} $\C=\{ C_i | 1 \le i \le k \}$ is a partition of ${\cal X}$ into $k$ groups. We define 
$$\maxdiamset(\C):=\max\{\diamset(C_i)|1 \le i \le k\} \, \, \mbox{ and }
\,\,  \avgdiamset(\C):=\frac{1}{k} \sum_{i=1}^k \diamset(C_i).$$
Moreover, let \OPTDIMK and \OPTAVGK be, respectively, the minimum possible  \maxdiamset \,  and \avgdiamset \, of a $k$-clustering
for $({\cal X},dist)$.


\noindent {\bf Arbitrary $k$.}
 First, in Section \ref{sec:first-bound}, we
 prove that for all $k$ the maximum diameter of the $k$-clustering
produced by \complink \, is at most
$k^{1.59} \OPTAVGK $. 
Since
$ \OPTAVGK \le \OPTDIMK$,
our result 
 is an improvement
 over  $O(k^{1.59}\OPTDIMK)$,  the best known  upper bound
on the maximum diameter of \complink \, \citep{arutyunova2023upper}.
Indeed, our bound can improve the previous one by up to a factor of $k$ 
since there are instances in which  $\OPTAVGK$ is $\Theta( \frac{\OPTDIMK}{k})$.


It is noteworthy that by using $\OPTAVG$ rather than $\OPTDIM$,
we can corroborate with the intuition that \complink \, produces
clusters with smaller diameters than those produced by \singlelink \,
since,
in addition to the  $k^{1.59} \OPTAVGK$ upper bound for the former,
we show an instance in which the maximum diameter of the latter is
$\Omega(k^2 \OPTAVGK)$. 
When  $\OPTDIM$ is employed, unexpectedly, as pointed out in \citep{arutyunova2023upper}, this separation is not possible
since the maximum diameter of \complink \, is 
  $\Omega(k \OPTDIMK)$ while that of \singlelink \, is $\Theta(k \OPTDIMK)$. 

To obtain the aforementioned upper bound, our main technique consists of carefully defining a partition 
of the clusters built by \complink \,  along its execution and then bounding
the diameter of the families in the partition. This technique
yields an arguably simpler analysis than that of \citep{arutyunova_et_al:LIPIcs.APPROX/RANDOM.2021.18,arutyunova2023upper}.

\remove{Moreover, it allows to show that $O(k^{1.59} \OPTAVG )$   is also an upper bound
on the  average distance among the points  of every cluster in the 
$k$-clustering produced by \avglink.	
From the best of our knowledge,
our bound  is the first one with respect to a traditional
cost function for this popular 
method.
}

Next, in Section \ref{sec:better-bound}, by using our technique in a significantly more involved way, we show that the maximum diameter of the $k$-clustering
produced by \complink \, is at most 
$(2k-2) \OPTDIMK$ for $k \le 4$
and at most $k^{1.30} \OPTDIMK,$ for $k > 4$.
Thus, we considerably narrow the gap
between the current upper bound and $\Omega(k \OPTDIMK ),$ the best known lower bound.

Finally, in Section \ref{sec:other-link}, we show that
our  techniques can be employed  to obtain upper
bounds on cohesion criteria 
of the clustering built by methods that  belong to a class of linkage methods that
includes \avglink \, and \minimax.
In particular, we show that the average
pairwise distance of every cluster in  the $k$-clustering
produced  by \avglink \, is at most $k^{1,59} \OPTAVGK$. 
To the best of our knowledge,
our analysis of the \avglink \, 
is the first one regarding to a cohesion criterion.

\noindent {\bf Low values of $k$ and practical applications.}
For large $k$,  the upper bounds of \complink, though close to the lower bound, are high and, thus, are not informative in the context of practical applications. However, as argued below, we have a different scenario for the very relevant case in which $k$ is small. The relevance of small $k$
 is that, in general, people have difficulties in analyzing a partition containing many groups (large $k$).

\citep{DBLP:journals/siamcomp/CharikarCFM04,DASGUPTA2005555} propose algorithms that obtain a hierarchical clustering  
that guarantees an 8-approximation to the diameter for every $k$.
The analysis  from
\citep{arutyunova2023upper} give, respectively, the following upper bounds on the approximation factor of 
 \singlelink \, and \complink \,
regarding the diameter: $4$ and $3$ for $k=2$; $6$ and $5.71$, for $k=3$
and $8$ and $9$ for $k=4$.
Our analysis gives 
an approximation factor 
of $2k-2$ for $k \le 4$, 
which improves these bounds.
 For an assessment of the quality of these bounds, one should take into account that, unless $P=NP$, the problem of finding the $k$-clustering that minimizes the maximum diameter, for $k \ge 3$, does not admit an approximation better than 2 in polytime \citep{DBLP:journals/jsc/Megiddo90}.

 For $k \ge 5$, the $k^{1.30} \OPTDIMK$ upper bound does not improve the factor of 8. However, our $k^{1.59} \OPTAVGK$ upper bound improves it for instances in which 
$\OPTAVGK \le \frac{8}{k^{1.59}}  \OPTDIMK$. Since $ \frac{\OPTDIMK}{k} \le \OPTAVGK \le
\OPTDIMK$ , we can have improvements for $k \le 34$.

An interesting aspect of our results
is that they point in the opposite direction
of the common intuition that
bottom-up methods for hierarchical clustering do not work well
for small $k$ and, hence, are 
less preferable than bottom-up methods.

\remove{
There exists  an intuition that
top-down methods are preferable for small
$k$, since it is not expected 
that bottom-up methods work well in this setting. Our theoretical results
point in the opposite direction by 
providing evidence that the performance
of bottom-up methods is good for small $k$.
}

\remove{We note that the quality of an approximation bound (e.g. factor 8 or 10) should be evaluated in light of the result from [Arutyunova and Röglin. 2022] which shows that no hierarchical compatible clustering can achieve an approximation smaller than 5.82 w.r.t the diameter (simultaneously for every $k$).
}

\subsection{Related Work}
Linkage methods are discussed  in a number of research papers and books on data mining
and machine learning.
Here, we discuss  works that provide provable guarantees for 
some of the most popular linkage methods.

\noindent {\bf Complete-link and Variants.} Several upper and lower bounds are known on the approximation factor for \complink \, with respect to the maximum diameter. 
When ${\cal X}= \mathbb{R}^d$, $d$ is constant and $dist$ is the Euclidean metric, 
\citep{DBLP:journals/corr/abs-1012-3697} proved that \complink \, is 
an $O(\log k \cdot \OPTDIMK)$ approximation.
This was improved by \citep{DBLP:conf/esa/GrosswendtR15} to $O(\OPTDIMK)$.
The dependence on $d$ is doubly exponential.

For general metric spaces,
\citep{DASGUPTA2005555} showed that there are instances
for which the maximum diameter
of the    $k$-clustering built by \complink \, is  $\Omega(\log k \cdot \OPTDIMK)$. 
In  \citep{arutyunova_et_al:LIPIcs.APPROX/RANDOM.2021.18} this lower bound
was improved to $\Omega(k \cdot \OPTDIMK)$. Moreover, the same paper
showed that the maximum diameter of
\complink's  \, $k$-clustering is  $O(k^{2} \OPTDIMK)$.  This result was recently improved by the same authors to 
$O(k^{1.59} \OPTDIMK)$ \citep{arutyunova2023upper}.  
We note that the version of \complink,  analyzed in \citep{arutyunova_et_al:LIPIcs.APPROX/RANDOM.2021.18,arutyunova2023upper},
merges at each iteration the two clusters $A$ and $B$ for which
$\diamset(A \cup B)$ is minimum. A consequence of Proposition \ref{prop:monotonic}, presented
here, is that this rule is equivalent to the classical definition of
\complink \, presented at the beginning of Section \ref{sec:preliminaries}.

\citep{arutyunova2023upper} also analysed \minimax \, \citep{BienTib2011}, a linkage method 
related to \complink, that merges at each iteration the two clusters $A$ and $B$
for which $A \cup B$ has the minimum ratio. They show that the
$\maxdiamset$ of the $k$-clustering built by  \minimax \, is  $\Theta(k \OPTDIMK)$.
In Section \ref{sec:other-link}, we show that
the $\maxdiamset$ is also $O(k^{1,59} \OPTAVGK)$.
One disadvantage of this method is that while \complink \, admits an $O(n^2)$ implementation
\citep{DBLP:journals/cj/Defays77}, no sub-cubic time implementation for minimax method 
is known \citep{BienTib2011}.

\noindent {\bf Single-link}. 
Among linkage methods, \singlelink \, is likely the one with the most extensive theoretical analysis \citep{DBLP:books/daglib/0015106,DASGUPTA2005555,arutyunova2023upper,LM23-Nips}.

The works of \citep{DASGUPTA2005555,arutyunova2023upper} are
those that are more related to ours.
The former  shows that  $\Omega( k  \cdot  \OPTDIMK)$ is a lower bound
on   the maximum diameter of {\tt Single-Link} while 
the latter proves that this bound is tight. 
We note that our $\Omega( k^2  \cdot  \OPTAVGK)$ lower bound improves over
that of \citep{DASGUPTA2005555} since $k \OPTAVGK \ge \OPTDIMK$. 
 
\remove{FINAL VERSION? It is well-known that \singlelink \, maximizes the minimum spacing
among different clusters \citep{DBLP:books/daglib/0015106}.
Recently, it was shown in \citep{LM23-Nips} that it also maximizes the minimum
spanning tree spacing, a criterion that is stronger than the maximum spacing.}



\remove{ 
 \citep{arutyunova_et_al:LIPIcs.APPROX/RANDOM.2021.18}
also considers the a variation of \complink \, where instead
of joining the two clusters $A$ and $B$ for which 
$\max_{x \in A} \max_{y \in B} dist(x,y)$ is minimized, it chooses the cluster $A$ and $B$
for which $\min_{x \in A} \max_{y \in B} dist(x,y)$ is minimized.
For this variant, known in the literature \citep{10.1093/bioinformatics/bti201} as minimax, they  proved a tight $\Theta(k \OPTDIM)$ bound.
}

\noindent {\bf Average-link}.
 \citep{DBLP:conf/stoc/Dasgupta16} introduced a global cost function defined over the tree
induced by a hierarchical clustering and proposed algorithms to optimize it.
 \citep{DBLP:journals/jacm/Cohen-AddadKMM19,DBLP:journals/jmlr/MoseleyW23} show that \avglink \, achieves constant approximation with respect to variants of the cost functions proposed by \citep{DBLP:conf/stoc/Dasgupta16}.
\citep{DBLP:conf/soda/CharikarCN19} proved that these analyses are tight.


\remove{\citep{DBLP:journals/jacm/Cohen-AddadKMM19} show that \avglink \, attains a 2-approximation for the cost function proposed by \citep{DBLP:conf/stoc/Dasgupta16} and the
 proximity between points of ${\cal X}$ are given by a dissimilarity measure.
In  \citep{DBLP:journals/jmlr/MoseleyW23}
introduced a cost function that can be seen as a dual of the one  proposed
in \citep{DBLP:conf/stoc/Dasgupta16}.  For this cost function they
 show that \avglink \, has  a constant factor optimization while
\complink\, and \singlelink\, have super-constant worst-case approximations.
In \citep{DBLP:conf/soda/CharikarCN19} it is shown that the bound of 
\citep{DBLP:journals/jmlr/MoseleyW23} for the \avglink \, is tight.
}

\noindent {\bf Ward}. Another popular linkage method was proposed by \citep{Ward63}.  
\citep{DBLP:conf/soda/GrosswendtRS19}
shows that Ward's method gives a 2-approximation for $k$-means when the optimal
clusters are well-separated.

\section{Preliminaries}
\label{sec:preliminaries}
Pseudo-code for \complink \, is shown in Algorithm \ref{alg:hac}. 
The function $\dist_{CL}(A,B)$ that measures the distance between clusters $A$ and $B$ is given by
$${\tt dist_{CL}}(A,B):=\max \{ dist(a,b)| (a,b)\in A \times B  \}.$$


\small
\begin{algorithm}
\small

  \caption{{\sc H\complink}(${\cal X}$,dist,dist$_{\L}$) }
   \begin{algorithmic}[1]

\STATE 
 $\C^{0} \gets$ clustering with $n$ unitary clusters, each one containing a point of
${\cal X}$

\STATE 
 {\bf For}  $i=1,\ldots,n-1$     
 \STATE  \hspace{0.2cm} $(A,B) \gets$ clusters in $\C_{i-1}$ s.t. ${\tt dist_{CL}}(A,B)$ is minimum 
 \STATE  \hspace{0.2cm} $\C^{i} \gets \C^{i-1}  \cup \{A \cup B\} - \{A,B\} $
   \end{algorithmic}
   \caption{Complete Link}
   \label{alg:hac}
\end{algorithm}

\normalsize

\normalsize

\remove{
We analyse certain algorithms 
from a natural  class of linkage algorithms  defined below.

\begin{definition}[well-behaved linkage algorithms] We say that a linkage algorithm $\L$ 
is {\em well-behaved} if it satisfies the following conditions:
(i) it can be implemented
through Algorithm \ref{alg:hac} and (ii) for every iteration $t$ of
every execution of $\L$,  if  $\L$ merges clusters $A$ and $B$ at iteration $t$, then 
 we have
\begin{equation}\min\{dist(a,b)|(a,b) \in A \times B\} \le  \diamset(C \cup D),
\label{eq:well-behaved-definition}
\end{equation}
for every two clusters $C$ and $D$ that are available to be merged at the beginning of   $t$.
\end{definition}

It is not difficult to see that \complink \, and \avglink \, are well-behaved.
The proof can be found in  Appendix \ref{sec:app-well-behaved}.
}


 The following property of \complink,
 whose proof can be found in the Appendix \ref{sec:monotonic}, will be useful for our analysis.
 In particular, it implies
 that the rule employed
 by \complink \, is equivalent
 to the rule analysed in \citep{arutyunova2023upper}
that merges at each iteration the two clusters $A$ and $B$ for which $\diamset(A \cup B)$ is minimum. 

\begin{proposition}
\label{prop:monotonic}
Let $A_j$ and $A'_j$ be the clusters merged
at  the $j$th iteration of \complink.
Then, for every $j \ge 1$, 
$$\diamset(A_j \cup A'_j) = \max \{ dist(x,y)| (x,y) \in A_j \times A'_j \}$$
and for every $j \ge 2$
$$\diamset(A_j \cup A'_j) \ge \diamset(A_{j-1} \cup A'_{j-1}).$$
\end{proposition}

We conclude this section with some useful
notation. The term \emph{family} is used to denote a set of clusters. 
For a family  $F$, 
we use $|F|$ and $\Pts(F)$, respectively,  to denote the number
of clusters in $F$ and the set of points that belong to some
cluster in $F$, that is, $\Pts(F)=\bigcup_{ g \in F} g$ . Moreover,  we use $\diamset(F)$ to denote 
the maximum distance between points that belong
to $\Pts(F)$.

\section{A first bound on the diameter of complete-link}
\label{sec:first-bound}

In this section, we prove that the maximum diameter of
the $k$-clustering built by \complink \, is at most $k^{1.59}\OPTAVGK$. 

\remove{
Our proof consists of keeping a dynamic partition of  the clusters produced by 
\complink \, into families and then bounding
the diameter of each family $F$
  as (essentially) a function of 
the clusters that $F$ touches in a  target $k$-clustering $\T=(T_1,\ldots,T_k)$.
 We note that our bounds will depend on the choice $\T$ and we
can take the best possible $\T$ according to our objective.
In this section, we use $\T$ as the $k$-clustering with minimum \avgdiamset.
}

Fix a target $k$-clustering $\T = (T_1, \ldots, T_k)$. Our proof maintains a dynamic partition of the clusters produced by \complink\, into families, where the diameter of each such family $F$ can be bounded in terms of the diameters of some of the $T_i$'s that it touches. We note that our bounds will depend on the choice of $\T$ and we
can take the best possible $\T$ according to our objective.
In this section, we take $\T$ to be the $k$-clustering with minimum \avgdiamset.

In Algorithm \ref{alg:FamilyGeneration1}, we define how the families evolve along the execution of \complink. 
  At the beginning, each of the $|{\cal X}|$ points is a cluster. 
We then define our first partition as $(F_1,\ldots,F_k)$, where  $F_i$ is a family that contains $|T_i|$ clusters, each one being a point from
$T_i$. Along the algorithm's execution, the families are organized in a directed forest $D$. Initially, the forest $D$ consists of $k$ isolated nodes,
where the $i$th node corresponds to family $F_i$.

When \complink \, merges the clusters  $g$ and $g'$ belonging
to the families $F$ and $F'$, respectively, a new
family $F^{new}$ is created and, in case (a) of Algorithm \ref{alg:FamilyGeneration1}, a second new family $F^{new'}$ is also created. These new families contain all the clusters in $F$ and $F'$, except for
$g$ and $g'$ that are replaced by the cluster $g \cup g'$.
Moreover, $F^{new}$ and $F^{new'}$ (when it is created)
become parents of $F$ and $F'$ in $D$. The precise definition of the new families 
and how the forest $D$ is updated are given by cases $(a)$ and $(b)$ in Algorithm \ref{alg:FamilyGeneration1}.



\begin{algorithm*}
\small

  \caption{{\sc Partitioning the Clusters of \complink} }
   \begin{algorithmic}[1]

\STATE 
 Create a clustering $\C^{0}$ with $n$ unitary clusters, each one containing a point of
${\cal X}$

\STATE  $\T= \{T_i|1 \le i \le k\} \gets$   $k$-clustering that satisfies
$\avgdiamset(\T)=\OPTAVGK$ \label{line:target} \\

\STATE   $F_i \gets \{ \{ x\} | x \in T_i\}, \, \forall i$ \\

\STATE  $D \gets $ forest comprised of $k$ isolated nodes $F_1,\ldots,F_k$. \\

\STATE  {\bf For} $t:=1,\ldots,n-k$ \\

\STATE \hspace{1cm}	 $(g,g') \gets$ next clusters to be merged by \complink \\

\STATE \hspace{1cm}	 $\C^{t} \gets \C^{t-1} \cup \{ g \cup g'\} - \{g, g'\} $ \\
 
\STATE \hspace{1cm}	Let $F$ and $F'$ be the families associated with the roots of $D$
 that respectively contain $g$ and $g'$. Assume w.l.o.g. $|F| \ge |F'|$. \\

\STATE \hspace{1cm} Proceed according to the following exclusive cases: \\

\STATE \hspace{1cm} ({\bf case a}) $|F'|=1$ and $|F|>1$  \\

\STATE \hspace{2cm} $F^{new} \gets F-\{g\}$; $F^{new'} \gets \{g \cup g'\}$ \\

\STATE \hspace{2cm} $F$.parent $\gets$ $F^{new}$ ;  $F'$.parent $\gets$ $F^{new'}$\\

\STATE \hspace{1cm} ({\bf case b})  $|F'|>1$ or $|F|=1$ \\

\STATE \hspace{2cm} $F^{new} \gets (F \cup F' \cup \{g \cup g'\})-g -g'$ \\

\STATE \hspace{2cm} $F$.parent $\gets$ $F^{new}$; $F'$.parent $\gets$ $F^{new}$ \\

  \end{algorithmic}
\label{alg:FamilyGeneration1}
\end{algorithm*}

\normalsize

To prove our bound,
we first show (Proposition \ref{prop:reg-family-number}) that at the beginning
of each iteration, there exists a family, among
those associated with some root of $D$, that
contains at least two clusters.
Then, we show an upper bound (Proposition \ref{prop:diameter-expansion}) on the diameter of every family, with at least two clusters, created by Algorithm \ref{alg:FamilyGeneration1}.
Finally, in Theorem \ref{thm:main1},  
this last result is used to upper bound the diameter
of every cluster created by \complink, based on a simple idea: if a cluster $g \cup g'$ is created at iteration $t$
 and $H$  is a family  containing two clusters, say $h$ and $h'$, at the beginning
of $t$, then \complink \, rule guarantees that $\diamset(g \cup g') \le
\diamset(h \cup h') \le \diamset(H)$.

\remove{For our analysis, we need
some extra terminology.
Let 
 $\desc(F):=\{i | 1 \le i \le k \mbox{ and }  F_i \mbox{ is a descendant of } F \mbox{ in } D \}$,
  $\phi(F):=|\desc(F)|$
  and $\phisum(F):=\sum_{i \in \desc(F)} \diamset(F_i)$.   
Note that if a family $F^{new}$ is parent of both families $F$ and $F'$ in $D$
then $\phi(F^{new})=\phi(F)+\phi(F')$ and
$\phisum(F^{new})=\phisum(F)+\phisum(F')$    
}

For our analysis, we need
some extra terminology.
Let 
 $\leaves(F)$ be the set of leaves
 of the subtree of $D$ rooted at node/family $F$. We define
  $\phi(F):=|\leaves(F)|$
  and $\phisum(F):=\sum_{H \in \leaves(F)} \diamset(H)$.   
Note that if a family $F^{new}$ is parent of both families $F$ and $F'$ in $D$
then $\phi(F^{new})=\phi(F)+\phi(F')$ and
$\phisum(F^{new})=\phisum(F)+\phisum(F')$    
Moreover, we say that  a family $F$ is \emph{regular} if $|F|>1$ and it is a \emph{singleton}
if $|F|=1$.

\begin{proposition} 
 At the beginning of each iteration of Algorithm \ref{alg:FamilyGeneration1}, 
 at least one of the roots of $D$ corresponds to a regular family. 
\label{prop:reg-family-number}
\end{proposition}
\begin{proof}
Initially, the total number of roots of $D$
is $k$. Since the number of roots either decreases or
remains the same, the number of roots
at the beginning of each iteration is at most
$k$. At the beginning of iteration $t$, for $t \le n-k$,
the \complink\, clustering $\C^{t}$ has more than $k$ clusters, each of them belonging
to one family that is a root of $D$. Since the number of roots
is at most $k$, then there will be two different clusters associated
with the same root, so that this root corresponds to a regular family.
\end{proof}

\begin{proposition}
 At the beginning of each iteration 
of Algorithm  \ref{alg:FamilyGeneration1} the diameter of every regular family $F$ satisfies    $\diamset(F) \le \phisum(F) \cdot \phi(F)^{(\log_2 3)-1} \le k^{\log_2 3} \OPTAVGK$.
 \label{prop:diameter-expansion}
\end{proposition} 
\begin{proof}
We  have that $\phi(F) \le k$, Moreover,  
the choice of the target clustering
$\T$ ensures that $\phisum(F) \le k \OPTAVGK$. Hence, the inequality $\phisum(F)  \phi(F)^{(\log_2 3)-1} \le k^{\log_2 3} \OPTAVGK$ holds. Thus, we focus on the first inequality.

 The proof is by induction on the iteration of \complink\ (and, in parallel, of Algorithm~\ref{alg:FamilyGeneration1}). 
For every initial family $F_i$, $\phi(F_i)=1$ and $\phisum(F_i)=\diamset(F_i)$. 
Thus, for every $F_i$, 
$\diamset(F_i) \le \phisum(F_i) \phi(F_i)^{(\log_2 3) -1} $.

Let us assume by induction that the result at the beginning of iteration 
$t$. We consider what happens in  iteration $t$ according to the possible cases:

\noindent {\bf case (a)}. In this case, $F^{new'}$ is a singleton so 
we do not need to argue about it since the property is about regular
families. 
Moreover, we have that 
\begin{align*}
\diamset(F^{new})=\diamset(F-\{g\}) \le \diamset(F) \le \\
\phisum(F) \phi(F)^{\log_2 3 -1} = \phisum(F^{new}) \phi(F^{new})^{\log_2 3 -1},
\end{align*}
where the last inequality holds by induction and the last identity holds
because $\phisum(F^{new})= \phisum(F)$ and  $\phi(F^{new})= \phi(F)$.


\noindent {\bf case (b)} We split  the proof into 3 subcases:

\noindent {\bf subcase 1.} $|F|=1$ and $|F'|=1$.
In the case  $F^{new}=\{g \cup g' \}$, so it is a singleton and, thus,  there
is nothing to argue since the property is about regular families.

\noindent {\bf subcase 2.} $|F'|>1$ and $F$=$F'$.
In this case, we have
\begin{align*}
\diamset(F^{new})=\diamset(F) \le \\
\phisum(F) \phi(F)^{\log_2 3 -1} = \phisum(F^{new}) \phi(F^{new})^{\log_2 3 -1},
\end{align*}
where the inequality holds by induction and the last identity holds
because $\phisum(F^{new})= \phisum(F)$ and  $\phi(F^{new})= \phi(F)$.


\noindent {\bf subcase 3.} $|F'|>1$ and $F \ne F'$.
This case is the most interesting one. In this case, \complink \, creates
a new family $F^{new}$ by merging
two  clusters $g$ and $g'$ from two distinct  regular families $F$ and $F'$.
Let $a$ and $b$ be two farthest points in $\Pts(F^{new})$.
If $a,b \in \Pts(F)$ or $a,b \in \Pts(F')$ the result holds for $F^{new}$ since 
\begin{align*} \diamset(F^{new}) \le \max\{\diamset(F),\diamset(F')\} \le \\
 \max \{\phisum(F) \cdot \phi(F)^{\log_2 3 -1},\phisum(F') \cdot \phi(F')^{\log_2 3 -1} \} \le \\
\le \phisum(F^{new}) \phi(F^{new})^{\log_2 3 -1}  
\end{align*}

Let  
$a \in \Pts(F)$, $b \in \Pts(F')$. 
We can assume  w.l.o.g. that
$$\phisum(F')\cdot \phi(F')^{(\log_2 3) -1} \le \phisum(F)\cdot \phi(F)^{(\log_2 3) -1}.$$
Note that this assumption will not conflict with the assumption $|F|\ge |F'|$ that was made to facilitate the presentation of Algorithm \ref{alg:FamilyGeneration1}. Indeed, we do not use the assumption $|F|\ge |F'|$ in what follows.

Let
$a' \in g$  and  $b' \in g'$ be points that
satisfy $dist(a',b')=\min \{ dist(x,y)
| (x,y) \in g \times g' \}$.
Moreover, let $h$ and $h'$ be any two clusters in $F$.
We have that
\begin{align}
dist(a',b') \le \max\{dist(x,y)|(x,y) \in g \times g'\} \le \label{eq:7apr24-1} \\
\max\{dist(x,y)|(x,y) \in h \times h'\} \le \diamset(h \cup h') \le \\
\diamset(F), \label{eq:7apr24-3}
\end{align}
where the second inequality follows from \complink \, rule.

By symmetry we also have $dist(a',b') \le\diamset(F')$ and,
hence 
\begin{equation}
dist(a',b') \le \min\{\diamset(F),\diamset(F')\}
\label{eq:well-behaved}
\end{equation}

\remove{
Let
$a'$ be a point in $g$ and $b'$ be a point in $g'$
such that $dist(a',b')= \min\{dist(x,y)|(x,y) \in g \times g'\}$,
that is, $dist(a',b')$ is the closest distance among  points in $g$ and $g'$.
We can  show that 
\begin{equation} 
dist(a',b') \le \min\{\diamset(F),\diamset(F')\}.
\label{eq:well-behaved}
\end{equation}
In fact, pick  two clusters $h$ and $h'$ from $F$.
\red{Since  \link \, is well-behaved,} inequality
\ref{eq:well-behaved-definition} guarantees that
 $dist(a',b') \le \diamset(h \cup h') \le \diamset(F)$.
The same argument shows that  $dist(a',b') \le \diamset(F')$.
}

Consider the sequence of points $a,a',b',b$. It
follows from the triangle inequality that  
\begin{align}
\diamset(F^{new})=dist(a,b) \le \label{eq:0} \\ \
 dist(a,a')+dist(a',b')+dist(b',b) \leq \label{eq:1}\\ 
  \diamset(F) + \diamset(F')+\diamset(F') \le  \label{eq:2} \\ 
 \phisum(F) \phi(F)^{\log_2 3 -1} + 2\phisum(F') \phi(F')^{\log_2 3 -1} \le  \label{eq:3} \\
(\phisum(F) +\phisum(F')  )(\phi(F')+\phi(F))^{\log_2 3 -1}=  \label{eq:4}\\ 
\phisum(F^{new}) \phi(F^{new})^{\log_2 3 -1},
\end{align}
where inequality (\ref{eq:1}) follows from (\ref{eq:well-behaved}), inequality (\ref{eq:2}) follows from the inductive hypothesis,  inequality  (\ref{eq:3}) follows from  Proposition \ref{prop:calculations-avg} (with $a=\phisum(F)$,
$b=\phisum(F')$, $x=\phi(F)$ and $y=\phi(F')$)
 and (\ref{eq:4}) holds because $\phi(F^{new})=
\phi(F)+\phi(F')$ and $\phisum(F^{new})=\phisum(F)+\phisum(F')$.
 \end{proof}

\remove{
Let us assume w.l.o.g that 
$a \in F$, $b \in F'$. 
 We must have 
 \begin{equation}
 \diamset( g \cup g') \le \min \{\diamset(F),\diamset(F')\},
\label{eq:alg-inequality}
  \end{equation}
 
otherwise {\tt Complete-Link} would have  merged in this iteration
two clusters from $F$ or two clusters from $F'$.

We can assume  w.l.o.g. that
$$\phisum(F')\cdot \phi(F')^{(\log_2 3) -1} \le \phisum(F)\cdot \phi(F)^{(\log_2 3) -1}.$$

Let
$a'$ be a point in $g$ and $b'$ be a point in $g'$.
Consider the sequence of points $a,a',b',b$. It
follows from the triangle inequality that  
\begin{align}
\diamset(F^{new})=dist(a,b) \le dist(a,a')+dist(a',b')+dist(b',b) \leq \\ \diamset(F)+\diamset(g \cup g')+\diamset(F') \le 
  \diamset(F) + \diamset(F')+\diamset(F') \le \\ 
 \phisum(F) \phi(F)^{\log_2 3 -1} + 2 \phisum(F') \phi(F')^{\log_2 3 -1} \le \\
(\phisum(F) +\phisum(F')  )(\phi(F')+\phi(F))^{\log_2 3 -1}= \\ 
\phisum(F^{new}) \phi(F^{new})^{\log_2 3 -1},
\end{align}
where the inequality (2) follows from the inductive hypothesis, the inequality (3) follows from  Proposition \ref{prop:calculations-avg} (with $a=\phisum(F)$,
$b=\Sigma\mbox{-Diam}(F')$, $x=\phi(F)$ and $y=\phi(F')$)
 and the last identity holds because $\phi(F^{new})=
\phi(F)+\phi(F')$ and $\phisum(F^{new})=\phisum(F)+\phisum(F')$.}

Now, we state and prove the main result of this section.

\begin{theorem}
For every $k$, the maximum diameter of the $k$-clustering built by \complink \, is at most 
$  k^{\log_2 3} \OPTAVGK$.
\label{thm:main1}
\end{theorem}
\begin{proof}
We prove by induction
 on the iteration of \complink\ (and, in parallel, of Algorithm~\ref{alg:FamilyGeneration1})
that the diameter of 
each cluster created by \complink \, is
at most $k^{\log_2 3} \OPTAVGK$.
At the beginning, we have $n$ clusters,
each of them corresponding to a point, so that
for every initial cluster $A$,
$\diamset(A)=0 \le k^{\log_2 3} \OPTAVGK$.
We assume by induction that at the beginning
of iteration $t$ every cluster satisfies the desired property, 

Let $g$ and $g'$ be two clusters merged at iteration $t$.  By Proposition \ref{prop:reg-family-number}
there is a regular family $F$ at the beginning of the $t$-th iteration.
Let $h$ and $h'$ be two clusters in $F$.
Therefore,  
\begin{align}
\diamset(g \cup g') =  \label{eq:apr11-1}\\
\max\{ \diamset(g),\diamset(g'),\dist_{CL}(g,g') \} \le \label{eq:apr11-2} \\
\max\{ \diamset(g),\diamset(g'),\dist_{CL}(h,h') \} \le  \label{eq:apr11-3} \\
\max\{ \diamset(g),\diamset(g'),\diamset(h \cup h') \} \le \label{eq:apr11-4} \\
\max\{ \diamset(g),\diamset(g'),\diamset(F) \}  \le \label{eq:apr11-5} \\
 k^{1.59} \OPTAVGK, \label{eq:apr11-7}
    \end{align}
where the first inequality holds
due to the choice
of \complink \, and the last one
from  the induction hypothesis and Proposition \ref{prop:diameter-expansion}.
\end{proof}

\remove{

We prove that the diameter of the cluster created in the $t$-th iteration, for  every $t$, is 
at most $k^{\log_2 3} \OPTAVGK$.

Let $g$ and $g'$ be the clusters that are merged at iteration $t$.
By Proposition \ref{prop:monotonic} $\diamset(g \cup g')=\max \{dist(x,y)| (x,y) \in g \times g' \}$. 
By Proposition \ref{prop:reg-family-number}
there is a regular family $F$ at the beginning of the t-th iteration.
Let $h$ and $h'$ be two clusters in $F$.
Therefore,  
\begin{align}
\diamset(g \cup g') =\max \{dist(x,y)| (x,y) \in g \times g' \} \le \label{lin:thm1} \\
\max \{dist(x,y)| (x,y) \in h \times h' \}   \le \diamset(F) \le  \label{lin:thm2} \\
\phisum(F) \phi(F)^{\log_2 3-1} \le  k^{\log_2 3} \OPTAVGK, \label{lin:thm3}
\end{align}
where (\ref{lin:thm1}) follows from    \complink \, rule;
the bound on $\diamset(F)$ holds due to Proposition \ref{prop:diameter-expansion}
 and  the inequality in (\ref{lin:thm3}) follows 
because  $\phi(F) \le k$ and, due
to the choice of the target clustering $\T$, $\phisum(F) \le k \OPTAVGK$.
\end{proof}
}

\remove{Since $\phi(F) \le k$, $\phisum(F) \le k \OPTAVGK$ and   Proposition \ref{prop:diameter-expansion} assures that $\diamset(F) \le \phisum(F)  \cdot \phi(F)^{(\log_2 3)-1}$  we conclude that 
$ \diamset(g \cup g') \le    \cdot k^{\log_2 3} \OPTAVGK$}

\singlelink \, is a popular linkage method whose pseudo-code is obtained by replacing
${\tt dist_{CL}}$ with ${\tt dist_{SL}}$   in Algorithm \ref{alg:hac},
where
$$ {\tt dist_{SL}}(A,B):= \min \{dist(a,b)|(x,y) \in A \times B\},$$

The rule employed by \singlelink, in contrast to that
of \complink, is not greedy with respect
to the minimization of the diameter.
Thus, it is expected that the latter presents better bounds than the former.
However, perhaps surprisingly, this is not the case when we consider
approximation regarding to $\OPTDIM$ since
 the  maximum diameter of the latter is $\Omega( \OPTDIMK)$ while that
 of the former is  $\Theta(\OPTDIMK)$ \citep{arutyunova2023upper}.

The use of $\OPTAVG$, instead of \OPTDIM, allows a separation between 
\complink \, and \singlelink \, in terms of worst-case approximation.
In fact, Theorem \ref{thm:main1}
shows that the maximum diameter
of \complink \,  is at most  $  k^{1.59} \OPTAVGK$ while the next result shows that
the maximum diameter of  \singlelink \, is $\Omega(  k^2 \OPTAVGK)$.

\begin{theorem}
 \label{thm:singlelink}
 There is an instance in which the $k$-clustering produced by \singlelink\, includes a cluster of diameter $\Omega(k^2  \OPTAVGK)$.
\end{theorem}
\begin{proof}
We present 
a simple instance for which
the $k$-clustering produced by \singlelink \, has
a cluster of diameter $\Omega(k^2 \OPTAVGK)$.
Let $B$ be a large positive number and let us consider $k$ groups $G_1,\ldots,G_k$:
$G_1$ consists of 2 points $a$ and $b$, with $dist(a,b)=B$; the group $G_i$, for $1<i< k$ is a singleton,
containing only the point $x_i$; and 
the group $G_k$ consists of $k-1$ points $y_1,\ldots,y_{k-1}$,
with $dist(y_i,y_j)=B+\epsilon$ for all $i$ and $j$.
 
Moreover, we have that $dist(x_i,a)=dist(x_i,b)=(i-1)\times (B-\epsilon)$ 
for $i=2,\ldots,k-1$ and $dist(x_i,x_j)=(j-i)(B-\epsilon)$,
for $1 < i < j < k$.
Finally, the distance of any point in $G_k$ to a point outside $G_k$ is $2B$.

Note that $(G_1,\ldots,G_k)$ is a $k$-clustering   and
the average diameter of its clusters is $ (2B+\epsilon)/k$.
On the other hand, \singlelink \, builds the $k-$clustering
$(G_1 \cup \cdots \cup  G_{k-1}, \{y_1\}, \ldots, \{y_{k-1}\})$
and the cluster ($G_1 \cup \cdots \cup  G_{k-1}$) has 
diameter   $(k-1)(B-\epsilon)$.
\end{proof}

\section{A Better Bound for Complete-Link}
\label{sec:better-bound}

One of the  key ideas of the approach presented in the previous section
is to use the diameter of a regular family  to bound
the diameter of any cluster that is created.
Indeed, if at the beginning of  an iteration, there is a family $F$ with two
clusters, then the diameter of the cluster created at this iteration  is
at most $\diamset(F)$. However, Algorithm \ref{alg:FamilyGeneration1} and its analysis do not
 take full advantage of this idea. As an example, let us assume that 
 at the beginning of some 
iteration there are 3 regular families, say $F$, $F'$ and $F''$, 
with $\diamset(F) \ge \diamset(F') \ge \diamset(F'')$,
all of them corresponding to roots of $D$. 
If a cluster in $F$ is merged with one in $F'$ (case (b) of Algorithm \ref{alg:FamilyGeneration1}) then 
a new family is created and its diameter is used as a bound, which is 
not desirable since it is larger than
that of $F''$.

To obtain a better bound, instead of creating a new family whenever
clusters from different families, say $F$ and $F'$, are merged, we create an
edge between $F$ and $F'$ in a dynamic graph $G$ that keeps track of the merges
among different families. When a connected component of $G$ has
at most one family that can still be used as a bounding tool, we
replace all families in the component with a new family.
The motivation for doing so  is to use a better bound
as much as possible, which contrasts with the approach taken by 
Algorithm \ref{alg:FamilyGeneration1}.

This new approach is presented in Algorithm \ref{alg:Families2}. 
 The algorithm  maintains  a set of excluded clusters $\Ex$;
 clusters in this set are 
 never included in the families that the algorithm creates (line \ref{line:FC-creation}).
Moreover, it maintains
both a direct forest $D$ and
a graph $G$. 
Each  node of $G$ as well as each node of $D$ is associated with a family; an edge is created in $G$ between nodes/families $F$ and $F'$ if \complink  \, merges two clusters $g \notin \Ex$ and  $g' \notin \Ex$ 
that, respectively, contain
points from  $\Pts(F)$ and $\Pts(F')$. 
The graph and the forest may be updated at each iteration of Algorithm \ref{alg:Families2}.


Before giving extra details regarding Algorithm \ref{alg:Families2}, we need to explain the concept of
a pure cluster. A family $F$  is created (lines \ref{line:initial-families} and  \ref{line:FC-creation})
by specifying the clusters that it contains.
 If a cluster $g$ is one of them,
 we say that $g$ is {\em pure} w.r.t. $F$ or, alternatively,
 $F$ has the pure cluster $g$.
Moreover, if a cluster $g$ is obtained by merging
two clusters that are pure with respect to some family $F$
then $g$ is also pure w.r.t.  $F$. If a cluster is not
pure w.r.t. any family, we say that it is {\em non-pure}.  
We use $\pure_t(F)$ to denote the number of pure clusters w.r.t. $F$ that
belong to $\C^t$.
Note that if $pure_{t-1}(F) \ge 2$
then  $\diamset(F)$ is an upper bound on the diameter
of the cluster that is created at iteration $t$, so that
families with at least two pure clusters play a role similar
to that of regular families in the analysis of Algorithm \ref{alg:FamilyGeneration1}.

In contrast to Algorithm \ref{alg:FamilyGeneration1},
where each cluster belongs to one family,
in Algorithm \ref{alg:Families2} every cluster 
that does not belong to $\Ex$ is either pure w.r.t. some family $F$ in $G$ 
(this would be equivalent of belonging to $F$) or it is contained
in $ \bigcup_{H \in C} \Pts(H)$ for some connected component $C$ in $G$. 
For our analysis, we note that $\Pts(F)$, $\diamset(F)$ and $|F|$ refer,
respectively, to the set of points of $F$, the diameter of $F$
and the number of clusters in $F$ {\bf at the moment} that $F$ is created.

The algorithm starts (lines \ref{line:t0}-\ref{line:end-init}) with the initialization
of the set $\Ex$, the forest $D$ and the graph $G$.
Then, in the  loop, two clusters are merged
following the \complink \, rule.
  In terms of the graph, each merge may lead to the
  addition of new edges and also to the
union of two connected components.
In terms of the families,
a merge can reduce by one unit the number of
pure clusters of one or two families.
If this happens pure clusters 
may be added to set $\Ex$ (lines \ref{line:addLr1} and \ref{line:addLr2}) and this may also trigger one of the cases (a), (b) or (c).
If either (a) or (b) occurs a new family $F_C$, associated with the component $C$
that satisfies one of these cases, is created
to replace all families in $C$ (line \ref{line:FC-creation}).
If case (c) occurs the component $C$ is removed from $G$.

The main loop was carefully designed to guarantee that (i) at the beginning of each iteration there exists a family that has at least two pure clusters associated with it and (ii)  the diameter of  family $F_C$
 is slightly smaller than twice the sum of the diameters of the families in the 
 underlying connected component $C$.

\begin{algorithm*}
\small

  \caption{{\sc Tighter Bound for \complink}  }
   \begin{algorithmic}[1]

\STATE $\C^{0} \gets$ clustering with $n$ unitary clusters, each one containing a point of
${\cal X}$ \label{line:t0} \\


\STATE $(T^*_1,\ldots,T^*_k) \gets$ a $k$-clustering with maximum diameter equal to \OPTDIMK  \\


\STATE  For each $i$, with $|T^*_i| > 1$,  $F_i \gets \{ \{x\} | x \in T^*_i\}$  \label{line:initial-families}\\


\STATE Create a forest $D$ with no edges and  vertex set  $\{F_i|T^*_i \mbox{ has  at least two points}\}$  \\

\STATE Create a graph $G$ with no edges and  vertex set  $\{F_i|T^*_i \mbox{ has  at least two points}\}$     \label{line:end-init} \\

\STATE $\Ex \gets $ set of clusters $T^*_i$ with exactly one point  \label{line:test} \\

\STATE  {\bf For} $t:= 1 \ldots n-k$  \\


\STATE	\hspace{0.5cm} $(g,g') \gets$ next clusters to be merged by {\tt Complete-Link} \\

\STATE	\hspace{0.5cm} $\C^{t} \gets \C^{t-1} \cup \{g \cup g'\} - \{g,g'\}$


\STATE \hspace{0.5cm} {\bf If}  $g$ or $g'$ is a cluster in $\Ex$ \\
  
\STATE \hspace{1cm} Add $g \cup g'$ to $\Ex$ and  remove from $\Ex$ the clusters in $\{g,g'\}$ that belong to $\Ex$   \\ \label{line:additionL-1}

\STATE \hspace{0.5cm} {\bf Else} 

\STATE \hspace{1cm} Create edges between all families $F$ and $F'$ such that $\Pts(F)$ has a point in $g$ and $\Pts(F')$ has a point in $g'$ \label{line:merge} \\
 
\STATE \hspace{0.5cm} Consider the following exclusive cases: \\

\STATE \hspace{0.5cm} (a) $\exists$ connected component $C$ in  $G$, with $|C|>1$, that has exactly one  family $F$ such that $\pure_t(F) >1$ 
 \label{line:case-a} \\

\STATE \hspace{0.5cm}  (b) $\exists$ connected component $C$ in  $G$, with $|C|>1$, such
that every family $F$ in $C$ satisfies $\pure_t(F) \le 1$  \label{line:case-b}
 \\

\STATE \hspace{0.5cm} (c) $\exists$ connected component $C$ in $G$, with $|C|=1$, and its only  family $F$
satisfies $\pure_t(F)  \le 1$\\

\STATE \hspace{0.5cm} {\bf If} $(b)$ does not  occur \label{line:block1}\\

\STATE \hspace{1cm} For each  family $H$ in $G$  that
satisifies $\pure_{t-1}(H)>1$ and $\pure_{t}(H)=1$    \\

\STATE \hspace{1.5cm}   Add the pure cluster in $H$ to $\Ex$ \label{line:addLr1} \\

\STATE \hspace{0.5cm}  {\bf If} (b) occurs \label{line:block2} \\

\STATE \hspace{1cm} $H \gets$ some family in $C$ such that
$\pure_{t-1}(H)>1 \mbox{ and } \pure_{t}(H)=1 $\\

\STATE \hspace{1cm}  Add the pure cluster in $H$ to $\Ex$ \label{line:addLr2}\\

\STATE \hspace{0.5cm}  {\bf If} either $(a)$ or $(b)$ occurs 
\label{line:block3-beg} \\

 
\STATE \hspace{1cm} Create family  
$F_C:=\{h | h \in \C^t \mbox{ and } h \subseteq \bigcup_{H \in C} \Pts(H)\} \setminus
\Ex$ \label{line:FC-creation} \\

\STATE \hspace{1cm} Set $F_C$ as the parent, in the forest $D$, of  every  family of $C$ 
\label{line:parent}  
\\

\STATE \hspace{1cm} Add to $G$ a node corresponding to $F_C$ \\

\STATE \hspace{1cm} Remove all families in the connected component $C$ from $G$

\STATE \hspace{0.5cm}  {\bf If} $(c)$  occurs \\

\STATE \hspace{1cm} Remove all families in the connected component $C$ from $G$
\label{line:block4-end}
\\
  \end{algorithmic}
\label{alg:Families2}

\end{algorithm*}

\normalsize



The roadmap to establish our improved bound (Theorem \ref{thm:better-bound}) consists of 
first showing  that (i) holds (Lemma \ref{lem:2pureclusters}) and, then, 
showing an upper bound on the diameters of the families $F_C$ that are created in line \ref{line:FC-creation}.
This upper bound will be used to bound the diameter of every cluster that is
created by \complink. 
Note that our strategy is similar to that employed to prove Theorem \ref{thm:main1}.
However, the proofs here are significantly more involved.

We start with Lemma \ref{lem:2pureclusters}. We present an overview of the proof
and we refer to Appendix \ref{sec:app-lem:2pureclusters} for the
full proof.

\begin{lemma} 
 For $t \le n-k$, at the beginning of iteration $t$ of Algorithm \ref{alg:Families2},
 each connected component $C$ of $G$ satisfies
one of the following properties:
(i) $|C|=1$ and the only family of $C$ has at least
two pure clusters  or (ii) $|C|>1$ and there exist  two  families in $C$
such that
each of them has at least two pure clusters.
\label{lem:2pureclusters}
\end{lemma}
\begin{proof}[Proof Sketch]
We first argue that if all components of $G$
satisfy the desired properties at the beginning of iteration $t$
then all components of $G$ also satisfy them at the beginning of  iteration $t+1$.
Next, we argue that  $G$ does not have all its nodes removed
at some iteration.

At the beginning of Algorithm \ref{alg:Families2}, all  the components in $G$ satisfy 
property (i) because, by line \ref{line:end-init}, all the families $F_i$ in $G$ have  
at least two clusters and all their clusters are pure. 

When two clusters are merged at some iteration $t$, then at most two
distinct families have their number of pure clusters
decreased by one unit (Proposition \ref{prop:families-evol}).  As a result, one connected component, say $C$, 
where these families lie in the updated graph may not respect the conditions of the lemma anymore.
However, in this case, we can show that either (a), (b) or (c) occurs. In the
case (c), the component  $C$ is removed from $G$, so we do not have a problem
with $C$ at the next iteration. If either (a) or (b) occurs,
$C$ is replaced with a new component that only has the family $F_C$. 
Proposition \ref{prop:FC} shows 
that there are two pure clusters w.r.t. $F_C$,  so this new component satisfies the condition (i). 

Now, assume that $G$ has all its nodes removed at
some iteration $t'$.
It is possible to conclude that at the beginning of $t'$,  $G$ has just one component, this component has just one family and this family has exactly 2 pure clusters. This together with the fact that at most $k$ clusters
are added to $\Ex$ (Proposition \ref{prop:L-Bound}) allows the conclusion that there are at most $k+1$ clusters at the beginning of $t'$. But this is not a problem since $t' \ge n-k$ in this case. 
\end{proof}

Now, we bound the diameter of the families $F_C$ at the moment they are created by Algorithm \ref{alg:Families2}.
To this end, we define a spanning tree $T_C$ for
 $C$ and use its paths to bound the diameter of $F_C$.
 Consider the sequence
of merges $m_1,\ldots,m_{|C|-1}$, between clusters,
 that builds the connected component
$C$, that is, right after each merge at least
two families in $C$ that were not connected become connected.
Moreover, let $g_i$ be cluster produced by merging $m_i$.
The nodes of $T_C$ are the families in $C$ and the edges of $T_C$ are defined  as follows:
for each merge $m_i$ we create an edge $e_i$ between two arbitrarily chosen families, say $F^1$ and $F^2$, among those that were not connected before merge $m_i$ and also
have points in $g_i$, that is, 
$\Pts(F^1) \cap g_i \ne \emptyset$ and $\Pts(F^2) \cap g_i \ne \emptyset$.
The weight of $e_i$ is given by the diameter of $g_i$.

For the following results, let \Diam$_i$ be the $i$th smallest diameter among the families
 that belong to $C$.

\remove{
The proof of  Proposition \ref{prop:spanning-tree} can 
be found in the appendix \ref{sec:proof:spanning-tree}. The key observation is that
the weight of $e_i$ is not larger than the diameter of families
that have at least two pure clusters right before  merge $m_i$
and it is also not larger than the diameter of families in $C$ that have not
been created when the merge $m_i$ occurs. By arguing that there
at least $|C|-i+2$ families that satisfy one of these conditions, we establish the proof.
}

\begin{proposition}
\label{prop:spanning-tree}
The weight of the cheapest edge of $T_C$ is at most \Diam$_{1}$ and,
for $i>1$, the weight of its $i$th cheapest edge is at most \Diam$_{i-1}$.
\end{proposition}
\begin{proof}

Let $F$ be a family that has at least two pure clusters right before the merge $m_i$
and let $h$ and $h'$ be two pure clusters w.r.t. $F$.
We first note that 
$$\diamset(g_i) \le \max\{dist(x,y)|(x,y) \in h \times h' \} \le \diamset(h \cup h') \le \diamset(F),$$
where the first inequality holds due to
the choice of {\tt Complete-Link} and Proposition \ref{prop:monotonic}.
Moreover, $$\diamset(g_i) \le \diamset(F')$$ for any family in $C$
that does not exist before the merge. In fact,
by condition (i) of  Lemma \ref{lem:2pureclusters}
$F'$ is created with at least two clusters, say $h$ and $h'$, and
$\diamset(g_i) \le \diamset(h \cup h') \le \diamset(F')$, where  
the first inequality follows from Proposition \ref{prop:monotonic}.
Hence, we can conclude that $\diamset(g_1) \le \Diam_1$ because before
the first merging each family in $C$ that already exists
is an isolated node in $G$ and has at least two pure clusters (condition (i) of  Lemma \ref{lem:2pureclusters}).

Now, we consider the case $i>1$. 
 Let $a$ be the number of families in $C$ that have at least two pure clusters right before the merge $m_i$
and let $b$ be the number of families in $C$ that have not been created yet. 
It is enough to show that $a+b \ge |C|-i+2$ (claim below).
In fact, in this case $|C|-i+2$ families in $C$ have
diameter not smaller than $\diamset(g_i)$ so that $\diamset(g_i) \le \Diam_{i-1}$ 

\noindent{\bf Claim.} $a+b \ge |C|-i+2$

\noindent {\it Proof}. Right before $m_i$, the families in $C$ are distributed in $(|C|-b)-i+1$ connected
 components in the graph $G$.
If one of these components has just one family, it follows from condition (i) of Lemma \ref{lem:2pureclusters} that this family must have at least two pure clusters.
If one of these components has at least two families, then
it follows from condition (ii) of Lemma \ref{lem:2pureclusters}
that there are two families in this component, each of them with at least
two pure clusters.

Since $i>1$,  at least one component has at least two families. Thus, there
are at least $(|C|-b)-i+2$ families with at least two pure clusters right before $m_i$.
We conclude that  $a \ge (|C|-b)-i+2$ and, hence,  
$a+b  \ge |C|-i+2$

\noindent {\it End of Proof}.
\end{proof}

The next proposition gives an upper bound
on the diameter of $F_C$ as a
function of the diameters
of the families in the component $C$ associated with $F_C$.
In high-level, its proof  considers
the path $P$ in $T_C$ between the families where the two farthest points in $\Pts(F_C)$ lie
and then use the triangle inequality  to show that
the distance between these points is upper bounded by the sum
of the weights of the edges in $P$ plus the sum of the diameters
of the nodes/families in $P$. This
sum, however, is  upper bounded by the sum of the diameters
of all the families in $C$ plus the sum of the weights of the edges in $T_C$,
so that
$$\diamset(F_C) \le \sum_{i=1}^{|C|} \Diam_i+\left( \Diam_1+ \sum_{i=1}^{|C|-2} \Diam_i\right ).$$
The proposition, in fact,  shaves $\Diam_1$ from the above upper bound via a more careful analysis.

\begin{proposition}
\label{prop:sum-diam}
Let $F_C$ be a family associated with the connected
component $C$ of $G$ in line \ref{line:FC-creation} of Algorithm \ref{alg:Families2}.
Then, when $F_C$ is created, we have 
$$ \diamset(F_C) \le \sum_{i=1}^{|C|} \Diam_i + \sum_{i=1}^{|C|-2} \Diam_i$$
\end{proposition}
\begin{proof}
For a given point $x$, we use $F_x$ to denote the family in connected
component $C$ where $x$ 
 lies right before the families in $C$ are replaced with $F_C$.
Let $a$ and $b$ be the  two farthest points of
$F_C$.
 We split the proof into two cases:

\noindent {\bf Case i)}
The path from the family  $F_a$ to $F_b$ in $T_C$
has less than $|C|-1$ edges.

Let  $u_1,\ldots,u_t$, with  $u_1=F_a$
 and $u_t=F_b$, be such a path. Note that $t <|C|$. 
Recall that in the construction of  $T_C$, an edge between families $u_i$ and $u_{u+1}$
is associated with some cluster $g$. Let 
$p'_i$ and $p_{i+1}$ be, respectively, arbitrarily chosen 
points in $\Pts(u_i)$ and $\Pts(u_{i+1})$ that belong to $g$.
 Now, consider the sequence of points
$(a=p_1,p'_1,p_2,p'_2,\ldots,p_t, p'_t=b)$.
We have that 
\begin{align*}
\diamset(F_C)= dist(a,b) \le \\ 
\sum_{i=1}^t dist(p_i,p'_i)+ \sum_{i=1}^{t-1} dist(p'_i,p_{i+1}) \le \\
 \sum_{i=2}^{|C|} \Diam_i +  \sum_{i=1}^{|C|-2} \Diam_i, 
\end{align*}
where the first  inequality holds due to the triangle
inequality and for the  second one we use the fact that
$\sum_{i=1}^t dist(p_i,p'_i)$ can be upper bounded
by the $|C|-1$ largest diameters of the families in $C$ and
 Proposition \ref{prop:spanning-tree} assures that 
 $\sum_{i=1}^{t-1} dist(p'_i,p_{i+1})$
can be upper bounded by the sum of the weights
of the $
|C|-2$ most expensive edges of $T_C$.

\noindent {\bf Case ii)}
The path from $F_a$ to $F_b$ in $T_C$
has  $|C|-1$ edges.

Since $|C|>1$ we have that $F_a \ne F_b$.
It follows from Proposition \ref{prop:LS-addition} that  there is $y \in \{a,b\}$ such that
a pure cluster w.r.t. family $F_y$  is added to $\Ex$, before the creation of $F_C$,
by either line  \ref{line:addLr1}
or line  \ref{line:addLr2}.

We assume w.l.o.g. that $y=a$. Let
$g$ be the pure cluster  w.r.t. $F_a$ that is added to $\Ex$.
We assume that $F_C$ is created at iteration $t$
and the addition of $g$ to $\Ex$ happened at iteration $t'$, so that
$t' \le t$. 
We cannot have $a \in g$ because points that belong to clusters in $\Ex$
are not in $\Pts(F_C)$. Moreover, $a$ cannot be in a pure cluster w.r.t. $F_a$ after
the $t'$-th merge,  otherwise we would  have $\pure_{t'}(F_a)\ge 2$ and
$g$ would not have been added to $\Ex$. Thus, right after the $t'$-th merge,     
 $a$  belongs to a cluster that contains a point, say $x$, from a
 family $F_x$ different from $F_a$.

\remove{\noindent {\it Claim.} There is $y \in \{a,b\}$ that
satisfies the following: $y$ does not lie in a pure cluster when
$F_C$ is created and at some iteration, a pure cluster from
$F_y$  is added to $\Ex$.

\noindent {\it Proof of the Claim.}
By construction, when $|C|>1$, every family in $C$ but one has
a pure cluster  added to the set $\Ex$ of singleton families (lines \ref{line:addLr1}
and \ref{line:addLr2}).
Then, either $a$ or $b$ belongs to a family  whose a pure cluster was added to
$\Ex$. We assume w.l.o.g. that $a$ belongs to one of these families.
When a pure cluster of $F_a$ is added to $\Ex$, then, by design,
there is only one pure cluster in $F_a$ and $a$ does not belong to this cluster,
otherwise, it could not belong to $F_C$ (by definition $F_C$ does not contain clusters in $\Ex$). \\
\noindent {\it End of the Proof }

\medskip

Let us assume w.l.o.g. that $y=a$ in the above claim. 
At the iteration in which a pure cluster from $F_a$ is added to $\Ex$,  
 $a$  belongs to a cluster that contains a point, say $x$, from a
regular family $F_x$ different from $F_a$.
}

We must  have
\begin{equation}
\label{eq:dist-ineq}
 dist(a,x)  \le \diamset(F_a)
 \end{equation}
  since the cluster that contains $a$ and $x$
 was created when $F_a$ still had at least two pure clusters.

 Now consider the path ($F_x=v_1,\ldots,v_t=F_b$) from $F_x$ to $F_b$ in $T_C$.
 This path does not include $F_a$, otherwise the path from $F_a$ to
 $F_b$ would have at most $|C|-2$ edges, which is not possible since
 we are in case (ii).
If the edge in $ T_C$ that connects families $v_i$ to $v_{u+1}$ corresponds
 to cluster $g$ then choose 
$p'_i$ and $p_{i+1}$ as 
points in $v_i$ and $v_{i+1}$, respectively, that belong to $g$.
 Now, consider a sequence of points
$(a,p_1,p'_1,p_2,p'_2,\ldots,p_t, p'_t)$, where
$p_1=x$ and $p'_t=b$.
From the triangle inequality,  
$$dist(a,b) \le dist(a,x) + \sum_{i=1}^t dist(p_i,p'_i)+ \sum_{i=1}^{t-1} dist(p'_i,p_{i+1}) .$$
Moreover, we have 
$$ \sum_{i=1}^t dist(p_i,p'_i) \le  \sum_{i=1}^{|C|} \Diam_i - \diamset(F_a)$$ 
and due to Proposition \ref{prop:spanning-tree}
$$\sum_{i=1}^{t-1} dist(p'_i,p_{i+1}) \le \sum_{i=2}^{|C|-1} \Diam_{i-1},$$
Hence,
\begin{align*}
dist(a,b) \le \\ 
dist(a,x) - \diamset(F_a) + \sum_{i=1}^{|C|} \Diam_i  + \sum_{i=2}^{|C|-1} \Diam_{i-1} \le \\
 \sum_{i=1}^{|C|} \Diam_i +\sum_{i=2}^{|C|-1} \Diam_{i-1}= 
 \sum_{i=1}^{|C|} \Diam_i +\sum_{i=1}^{|C|-2} \Diam_{i}, 
 \end{align*}
where the last inequality follows from 
(\ref{eq:dist-ineq}).
\end{proof}

For the next lemma 
recall that   
$\phi(F)=|\leaves(F)|$, 
where $\leaves(F)$ is the set of  leaves in
the subtree of $D$ rooted at node/family $F$.

\remove{at initial families
that are descendant of $F$ in the directed forest $D$
maintained by Algorithm \ref{alg:FamilyGeneration1}.
The next lemma makes use of this definition but note  that in Algorithm \ref{alg:Families2}
only the initial families $F_i$ with at least two clusters are added to
$D$.
}
Let $\alpha= \max\{\frac{\log (2i-2)}{  \log i}  | i \mbox{ is a natural number larger than 1} \}$.
Proposition \ref{prop:calculo-alpha}  shows  that $\alpha = \frac{\log 6}{\log 4} < 1.30$.
Moreover, we define $\alpha_k=\log_k (2k-2)$, if $k\le 4$, and 
$\alpha_k=\alpha$ for $k >4$.


\begin{lemma}
Every family  $F$ created by 
Algorithm \ref{alg:Families2} satisfies
  $\diamset(F) \le \OPTDIMK \phi(F)^{\alpha_k}$.
\label{lem:main-bound-F}
\end{lemma}
\begin{proof}
The initial  families  $F_i$ satisfies 
the property 
because $\diamset(F_i)=\diamset(T^*_i) \le \OPTDIMK \le  \OPTDIMK  \phi(F_i)^{\alpha_k} $
since $\phi(F_i)=1$.

Let us assume that the result holds at the beginning of iteration $t$.
If no family is created at iteration $t$ the result holds
at the beginning of iteration $t+1$. Otherwise,  
 a  family $F_C$,
 associated with a connected component
$C$,  is created. Let  $\{F_C^{i}| i=1,  \ldots,|C|\}$ be the nodes/families
in $C$ right before the creation of $F_C$.
Moreover, assume that $\phi(F_C^{i} )\le \phi(F_C^{i+1} )$. 
We have that

\begin{align}
\diamset(F_c) \le \label{eq:04Jan1} \\
\sum_{i=1}^{|C|} \Diam_i + \sum_{i=1}^{|C|-2} \Diam_i \le \label{eq:04Jan2}\\
 \sum_{i=1}^{|C|} \diamset(F_C^{i} ) + \sum_{i=1}^{|C|-2} \diamset(F_C^{i} )  \le \label{eq:04Jan3}\\ 
 \OPTDIMK \cdot \left ( \sum_{i=1}^{|C|}  \phi(F_C^{i})^{\alpha_k} + \sum_{i=1}^{|C|-2}   \phi(F_C^{i})^{\alpha_k}  \right )\le \label{eq:04Jan4} \\
\OPTDIMK \left ( \sum_{i=1}^{|C|} \phi(F_C^i) \right )^{\alpha_k} = \OPTDIMK \phi(F_C)^{\alpha_k}, \label{eq:04Jan5} 
\end{align}
where (\ref{eq:04Jan1}) follows from  Proposition \ref{prop:sum-diam};
(\ref{eq:04Jan2}) holds because $\Diam_1,\ldots,\Diam_{|C|-2}$ are
the $|C|-2$ smallest diameters among
the diameters of the families in $C$; 
(\ref{eq:04Jan3}) follows from the inductive hypothesis and
(\ref{eq:04Jan4})  follows from  Proposition  \ref{prop:calculations2}, using
$a_i = \phi(F_C^i)$, $\ell=|C|$ and $p=\alpha_k.$
\end{proof}

\remove{\begin{align*}
diam(F_c) \le \sum_{i=1}^{|C|} \Diam_i + \sum_{i=1}^{|C|-2} \Diam_i  \le \\ 
 \sum_{F \in C} OPT \cdot \phi(F)^\alpha +  \sum_{F \in C' }  OPT \cdot  \phi(F)^\alpha.
\end{align*}
where $C'$ is the set of families in $C$, excluding the two families
with the largest diameter.
However,  
\begin{align*}
\sum_{F \in C} OPT \cdot \phi(F)^\alpha +  \sum_{F \in C' }  OPT \cdot  \phi(F)^\alpha 
\le \\
 OPT \left ( \sum_{F \in C} \phi(F) \right )^\alpha = OPT \phi(F_C)^\alpha  , 
\end{align*}

where the  inequality follows from  Proposition  \ref{prop:calculations2}
since $p \ge \log_{|C|} (2|C|-2)$.}

We can state the main result of this section,
its proof is very similar to that of Theorem \ref{thm:main1}.

\begin{theorem}
The maximum diameter among the clusters of the $k$-clustering produced
by \complink \, is at most 
$(2k-2)\OPTDIMK$, if $k \le 4$, and  at most $ k^{1.30}\OPTDIMK$, if $k >4$.
\label{thm:better-bound}
\end{theorem}
\begin{proof}
Due to the definition of $\alpha_k$,
it is enough to show that the diameter of every cluster created
by \complink \, is at most $\OPTDIMK \cdot k^{\alpha_k} $.

We prove it  by induction on
the number of iterations of the Algorithm
\ref{alg:Families2}.
At the beginning, all $n$ clusters have a diameter of 0, so the result holds.

We assume by induction that the result holds at the beginning
of iteration $t$.
At the beginning of this iteration, by Lemma \ref{lem:2pureclusters}
there is a family, say $F$, with at least 2 pure
clusters. Let  $h$ and $h'$ be these clusters. 
Moreover, let  $g$ and $g'$  be the clusters merged at iteration $t$.
 We have that
\begin{align*} 
\diamset(g \cup g') = \\
 \max \{\diamset(g),\diamset(g'),\dist_{CL}(g,g') \} \le \\
  \max \{\diamset(g),\diamset(g'),\dist_{CL}(h,h') \} \le \\
  \max \{\diamset(g),\diamset(g'),\diamset(h \cup h') \} \le \\
  \max \{\diamset(g),\diamset(g'),\diamset(F) \} \le \\
  \OPTDIMK k^{\alpha_k}
\end{align*}
where 
the first  inequality is due to the choice of \complink \, and the
last inequality holds due to Lemma \ref{lem:main-bound-F},
the inductive hypothesis and the fact that $\phi(k) \le k$.
\end{proof}

We note that $k=2$, $k=3$ and $k=4$, we get approximation 
factors of $2$, $4$ and $6$, respectively. For $k >4$ the approximation
factor is $k^{\log_4 6}
 \le k^{1.30}.$
\remove{
\red{We remark that for $k=2$ and $k=3$ we can use, respectively, $\alpha=\log_2 2 $ and $\alpha=\log_3 4$, rather
than $\alpha=\log_4 6$ in Lemma 
\ref{lem:main-bound-F}.
Thus, the upper bounds on  the approximation factor of \complink \, for $k=2$ and $k=3$ are improved to 2 and 4, respectively.}
}
\section{Other linkage methods}
\label{sec:other-link}

In this last  section,
we show that  Theorem \ref{thm:main1} generalizes
to a class of linkage methods that
includes  \minimax \, and the quite popular \avglink.

Let $f$
be a distance function
that maps a pair of clusters into a non-negative real number
and let {\tt Link}$_f$ be a linkage method that  follows the pseudo-code of
 Algorithm \ref{alg:hac}, with the exception that it uses the function $f$,  rather than {\tt dist$_{CL}$},   to measure the distance between two clusters.
 Moreover, for a cluster $A$, let $\cost(A)$ 
 be a cohesion criterion (e.g. diameter).
We say that   $f$ and 
$\cost$ {\em align} if they
 satisfy the following conditions
 for every pair of disjoint clusters $A$ and $B$:

 \medskip

(i) $\min\{\dist(a,b)| (a,b) \in A \times B \} \le f(A,B) \le \diamset(A \cup B)$;

(ii) $\cost(A)=0$ if $|A|=1$;

(iii) $\cost(A \cup B) \le \min\{\cost(A),\cost(B), f(A,B)\}$

\medskip
Theorem \ref{thm:general} presented below is a generalization of 
Theorem \ref{thm:main1}. In fact,
from the former, we can recover the latter by
setting $\cost=\diamset$
and $f=\dist_{CL}$.
The proof of Theorem \ref{thm:general} is 
essentially the same as that of Theorem \ref{thm:main1}, but for a few
differences that we explain in
what follows.

\remove{
following differences: condition (i), rather than \complink's rule, 
is used to prove inequalities (\ref{eq:7apr24-1})-(\ref{eq:7apr24-3}) in Proposition \ref{prop:diameter-expansion} and
both conditions (i) and (ii) are used to prove inequalities (\ref{lin:thm1})-(\ref{lin:thm3}) in Theorem \ref{thm:main1}.
}

The proof of Theorem \ref{thm:general} is
based on the analysis of the families generated by the variation of Algorithm \ref{alg:FamilyGeneration1} 
that uses a distance function $f$ 
that satisfies (i), rather than $\dist_{CL}$,
to decide which  clusters are merged
at each iteration.
We use {\tt Algo}$_f$ to denote this modified
version of Algorithm  \ref{alg:FamilyGeneration1}.

Proposition \ref{prop:reg-family-number}
does not depend on the distance
function employed to decide
which clusters shall be merged at each iteration, so it is still valid
for {\tt Algo}$_f$.

The following  proposition
generalizes Proposition  \ref{prop:diameter-expansion} for
linkage methods whose underlying distances
satisfy condition (i).

\begin{proposition} 
If $f$ satisfies condition (i),
then
 at the beginning of each iteration 
of {\tt Algo}$_f$  the diameter of every regular family $F$ satisfies 
$$\diamset(F)\le   \phisum(F) \cdot \phi(F)^{(\log_2 3)-1} \le k^{\log_2 3} \OPTAVGK.$$
 \label{prop:diameter-expansion-general}
\end{proposition}
\begin{proof}
In Proposition \ref{prop:diameter-expansion},
the \complink's rule is just used to
prove inequalities (\ref{eq:7apr24-1})-(\ref{eq:7apr24-3}). 
However, these inequalities
are valid if the function $f$
satisfies condition (i).
In fact, we have
$$ dist(a',b') \le f(g,g) \le f(h,h') 
\le \diamset(h \cup h') \le \\
\diamset(F),
$$
where the first and the third
inequality hold due to condition (i)
while the second holds due
to the choice of \link$_f$.
\end{proof}


\begin{theorem}
If $f$ and $\cost$ align, then 
the $k$-clustering $\C$ built by 
 {\tt Link}$_f$  satisfies
$$\max\{\cost(C)| C \in \C\} \le k^{1.59} \OPTAVGK$$
\label{thm:general}
\end{theorem}
\begin{proof}
We prove by induction on
the number of iterations
of \link$_f$ (in parallel on {\tt Algo}$_f$)
that  each cluster $A$ created by \link$_f$ satisfies $\cost(A) \le k^{\log_2 3} \OPTAVGK$.
At the beginning, this holds because
every cluster $A$ is a point,
so that $\cost(A)=0$ due to condition (ii).
We assume by induction that the desired
property holds
at the beginning of iteration $t$.

Let $g$ and $g'$ be two clusters merged at iteration $t$.
 By Proposition \ref{prop:reg-family-number}
there is a regular family $F$ at the beginning of the $t$-th iteration.
Let $h$ and $h'$ be two clusters in $F$.
Therefore,  
\begin{align}
\cost(g \cup g') \le  \\
\max\{ \cost(g),\cost(g'),f(g,g') \} \le \\
\max\{ \cost(g),\cost(g'),f(h,h') \} \le \\
\max\{ \cost(g),\cost(g'),\diamset(h \cup h') \} \le \\
\max\{ \cost(g),\cost(g'),\diamset(F) \}  \le \\
k^{1.59} \OPTAVGK 
 \end{align}
where the first inequality holds
due to condition (iii),
the second due to the choice
of \link$_{f}$,  the third
due to condition (i) and the last one follows
from induction and Proposition \ref{prop:diameter-expansion-general}.
   \end{proof}

Now, we specialize 
Theorem \ref{thm:general} for \avglink\, and \minimax.
\avglink \, employs
the distance function
$$\dist_{AL}(A,B)=\frac{1}{|A|\cdot |B|} \sum_{a \in A} \sum_{b \in B} dist(a,b)$$
to measure the distance between clusters $A$ and $B$.
Clearly, $\dist_{AL} $
satisfies condition (i).
For a  cluster $A$, we define  $\avg(A)$ as $0$  if $|A|=1$ and as the average pairwise distance of the points in  $A$ if $|A|>1$, that is, $$\avg(A):= \frac{2}{|A|(|A|-1)}\sum_{x,y \in A} dist(x,y).$$
Since $\avg(A \cup B)$ is a convex
combination of $\avg(A)$, $\avg(B)$ and $\dist_{AL}(A,B)$, condition (iii) is also satisfied and, therefore, $\dist_{AL}$ and $\avg$ align.  
We have the following result.

\begin{theorem}
For every $k$, the $k$-clustering $\C$ built by 
\avglink \, satisfies
$$\max\{\avg(C)| C \in \C\} \le k^{1.59} \OPTAVGK$$
\end{theorem}

Now we consider the  \minimax \, linkage method.
This method
employs the function 
$$\dist_{MM}(A,B):=\min_{ x \in A \cup B} \max_{y \in A \cup B} dist(x,y). $$ 
to measure the distance between clusters.

We have that $\dist_{MM}$ satisfies (i).
Consider the cohesion criterion
$\radius(A)$ that has value 0 if $|A|=1$ and
when $|A|>1$,
$$\radius(A):=
\min_{ x \in A } \max_{y \in A } dist(x,y).
$$
Since $\radius(A \cup B)=\dist_{MM}(A,B)$
the condition (iii) is also satisfied
and, hence, $\dist_{MM}$ and
$\radius$ align.
We have that
\begin{theorem}
For every $k$, the $k$-clustering $\C$ built by 
\minimax \, satisfies
$$\max\{\radius(C)| C \in \C\} \le k^{1.59} \OPTAVGK.$$
\end{theorem}

 \remove{
The reason  why the theorem holds is that the \complink's \, rule is only used to prove inequalities (\ref{eq:7apr24-1})-(\ref{eq:7apr24-3}) in Proposition \ref{prop:diameter-expansion} and inequalities (\ref{lin:thm1}) -(\ref{lin:thm3}) in Theorem \ref{thm:main1}.
 The former inequalities hold if (i) is valid and the latter hold if both (i) and (ii) are valid.
 }

\remove{
\section{Societal Consequences}
This paper presents work whose goal is to advance the field of Machine Learning. There are many potential societal consequences of our work, none which we feel must be specifically highlighted here.    
}


\remove{
\section{Conclusions}
We have presented a new and improved analysis for the 
\complink \, method.
We proved that for every $k$, the maximum diameter
of the $k$-clustering built by \complink \, is
$O( k^{1.30} \OPTDIMK)$, improving over 
$O( k^{1.59} \OPTDIMK)$, the previous
currently upper bound. In addition, we showed that
the maximum diameter is at most  $ k^{1.59} \OPTAVGK$,
which is also an improvement over the previous bound since
$\OPTAVGK \le \OPTDIMK$. 
The  analysis with regards to $\OPTAVGK$ allows
a separation between \complink \, and \singlelink \,
in terms of worst-case approximation,
what is not possible when $\OPTDIMK$ is considered.

We leave as an interesting open question the determination of  tight bounds
for the diameter  of \complink \, as a function of 
$\OPTAVGK$ and/or
$\OPTDIMK$
}

\bibliography{biblio}
\bibliographystyle{icml2024}

\onecolumn
\appendix

\section{Proof of Proposition 
\ref{prop:monotonic}}

In this section, we present the proof
of Proposition \ref{prop:monotonic} and then  we argue that  it implies  that the rule employed
 by \complink \, is equivalent
 to the rule that chooses
 at each iteration the two clusters
 $A$ and $B$ for which $\diamset(A \cup B)$ is minimum. This rule was analyzed in \citep{arutyunova2023upper}.

\label{sec:monotonic}
\begin{proof}
The proof is by induction. 
For $j=1$, $A_j$ and $A'_j$ are singletons so that
$\diamset(A_1 \cup A'_1)=dist(x,y) $, where $x$ and $y$
are the only points in $A_1$ and $A'_1$, respectively.

We assume by induction that the result holds for every $i<j$.
First we prove that $\diamset(A_j \cup A'_j) \ge \diamset(A_{j-1} \cup A'_{j-1} )$.
Note that 
\begin{equation}
\diamset(A_j \cup A'_j) = \max \{ \diamset(A_j), \diamset(A'_j),  
\max \{dist(x,y)|(x,y) \in A_j \times A'_j   \} \}
\label{eq:monotonic0}
\end{equation}
If $A_j= A_{j-1} \cup A'_{j-1}$ or $A'_j= A_{j-1} \cup A'_{j-1}$  we
conclude that $\diamset(A_j \cup A'_j) \ge \diamset(A_{j-1} \cup A'_{j-1} )$.
Otherwise, we have that 
\begin{align*}
 \diamset(A_j \cup A'_j) \ge \max\{dist(x,y) | (x,y) \in A_j \times A'_j \} \ge \\
\max\{dist(x,y) | (x,y) \in A_{j-1} \times A'_{j-1} \}= \diamset(A_{j-1} \cup A'_{j-1}),
\end{align*}
where the second inequality follows from the \complink \, rule
 and the last identity follows by induction.

It remains to show that 
\begin{equation}
\diamset(A_j \cup A'_j) =  
\max \{dist(x,y)|(x,y) \in A_j \times A'_j    \}
\label{eq:monotonic}
\end{equation}

First, we consider the case where either 
$A_j= A_{j-1} \cup A'_{j-1}$ or  $A'_j= A_{j-1} \cup A'_{j-1}$.
We assume w.l.o.g. that  $A_j= A_{j-1} \cup A'_{j-1}$.
Thus, $\diamset(A_j)=\diamset( A_{j-1} \cup A'_{j-1}) \ge \diamset(A'_j)$,
where the inequality holds 
because either $A'_j$ is a singleton or it was
obtained by merging two clusters before the iteration $j-1$ and,
in this case, the induction hypothesis guarantees the inequality.
Moreover, in this case,   
cluster $A'_j$ is available to be merged right before the $(j-1)$th merge, so 
 it follows from 
the \complink \, rule that
 
\begin{equation}
\label{eq:monotonic2}
\max \{dist(x,y)|(x,y) \in A'_{j} \times A_{j-1} \}
\ge \max \{dist(x,y)|(x,y) \in A_{j-1} \times A'_{j-1} \}
\end{equation} 
and 
\begin{equation}
\label{eq:monotonic3}
 \max \{dist(x,y)|(x,y) \in A'_{j} \times A'_{j-1} \}
\ge \max \{dist(x,y)|(x,y) \in A_{j-1} \times A'_{j-1} \}
\end{equation} 

Thus, 
\begin{align}
\diamset( A_{j} \cup
A'_{j} ) \ge   \label{line:monot-1} \\ 
 \max \{dist(x,y)|(x,y) \in A_{j} \times A'_{j} \}
= \label{line:monot0}\\
 \max \{dist(x,y)|(x,y) \in (A_{j-1} \cup A'_{j-1}) \times A'_{j} \}
\ge \label{line:monot05}\\
\max \{dist(x,y)|(x,y) \in A_{j-1} \times A'_{j-1} \} = \label{line:monot1} \\
\diamset(A_{j-1} \cup
A'_{j-1} ) = \\
 \diamset(A_{j}) \ge   \diamset(A'_{j}), \label{line:monot2}
\end{align}
where  (\ref{line:monot05}) follows from (\ref{eq:monotonic2}) and (\ref{eq:monotonic3}),
while
(\ref{line:monot1}) follows by induction.

Therefore, (\ref{eq:monotonic}) must hold,
otherwise the inequalities (\ref{line:monot-1})-(\ref{line:monot2})
 would contradict (\ref{eq:monotonic0}).

If neither $A_j= A_{j-1} \cup A'_{j-1}$ nor $A'_j= A_{j-1} \cup A'_{j-1}$  we
have that
\begin{align}
\diamset(A_j \cup A'_j) \ge \\
\max\{dist(x,y) | (x,y) \in A_{j} \times A'_{j} \} \ge \\ 
\max\{dist(x,y) | (x,y) \in A_{j-1} \times A'_{j-1} \} = \\ 
\diamset(A_{j-1} \cup A'_{j-1})  \ge \\
\max \{ \diamset(A_j), \diamset(A'_j) \} 
\end{align}
Again, (\ref{eq:monotonic}) must hold, otherwise we contradict (\ref{eq:monotonic0}).
\end{proof}


Now we show that   the rule employed
 by \complink \, is equivalent
 to the rule that chooses
 at each iteration the two clusters
 $A$ and $B$ for which $\diamset(A \cup B)$ is minimum. 

 We assume for the sake of reaching a contradiction that at some iteration 
  \complink\, merges clusters $A$ and
 $B$ while there were clusters $A'$ and $B'$, 
 with $\diamset(A' \cup B') < \diamset(A \cup B)$, that could be merged. In this case, we  conclude
 that 
 $\dist_{CL}(A',B') \le \diamset(A' \cup B')
< \diamset(A \cup B) = \dist_{CL}(A,B),$
where the last identity follows from Proposition \ref{prop:monotonic}.
However, this contradicts the choice of 
\complink.

\section{Proof of Lemma \ref{lem:2pureclusters}}
\label{sec:app-lem:2pureclusters}
The following propositions are helpful to prove
Lemma \ref{lem:2pureclusters}.
The first one
characterizes how $\pure_t$
evolves when two clusters are merged.

\begin{proposition} 
Let $g$ and $g'$ be the clusters merged at the iteration $t$  of Algorithm \ref{alg:Families2}.
Then, exactly one of the following cases happen:

\begin{enumerate}
\item Both $g$ and $g'$ are non-pure.
We have that 
 $\pure_{t}(H)=\pure_{t-1}(H)$  for every family $H$ in $G$.

\item $g$ is a pure cluster w.r.t. family $F$ and $g'$ is non-pure.
We have that $\pure_{t}(F)=\pure_{t-1}(F)-1$ and
 $\pure_{t}(H)=\pure_{t-1}(H)$  for every family $H \ne F$.

\item $g'$ is a pure cluster w.r.t. family $F'$ and $g$ is non-pure.
We have that  $\pure_{t}(F')=\pure_{t-1}(F')-1$ and
 $\pure_{t}(H)=\pure_{t-1}(H)$  for every family $H \ne F'$.

\item $g$ and $g'$ are  pure clusters w.r.t. families $F$ and $F'$, respectively. 

Then,  $\pure_{t}(F)=\pure_{t-1}(F)-1$,
 $\pure_{t}(F')=\pure_{t-1}(F')-1$
 and
 $\pure_{t}(H)=\pure_{t-1}(H)$  for every family $H \notin \{F,F'\}$.

Moreover, 
If $\pure_{t-1}(F) \ge 2$ 
and $\pure_{t-1}(F') \ge 2$ then 
$g \cup g'$ is not added to $\Ex$
by line \ref{line:additionL-1} of Algorithm \ref{alg:Families2}.

 
\end{enumerate}
\label{prop:families-evol}
\end{proposition}
\begin{proof}
We argue for each of the cases of the statement:

\begin{enumerate}

\item This case holds because no pure cluster is affected when $g$ and $g'$ are merged.

\item In this case, $g$ does not count for $\pure_t(F)$,
because $g$ is merged, and $g \cup g'$ is not a pure
cluster. Thus, $\pure_t(F)=\pure_{t-1}(F)-1$.

\item The proof of this case is analogous to that of item 2.

\item 
If $F=F'$  then $g \cup g'$ is pure w.r.t. $F$.
Thus, $\pure_t(F)=\pure_{t-1}(F)-1$ because
$g$ and $g'$ counts only for $\pure_{t-1}(F)$
while $g \cup g'$ just count for $\pure_{t}(F)$.

If $F \ne F'$  then $g \cup g'$ is not pure.
Since $g$ counts for $\pure_{t-1}(F)$ but not
for $\pure_t(F)$ we have $\pure_t(F)=\pure_{t-1}(F)-1$.
By using the same reasoning we conclude that
$\pure_t(F')=\pure_{t-1}(F')-1$.

If $\pure_{t-1}(F) \ge 2$ we cannot have $g \in \Ex$
because $g$ is pure w.r.t. to $F$ and no
cluster in $\Ex$ is pure with respect to $F$.
In fact, any cluster $h \in \Ex$
is added to $\Ex$  
by either line \ref{line:addLr1} or line \ref{line:addLr2},
or $h=h' \cup (h-h')$, where $h'$ is a cluster
that was added to $\Ex$  
by either line \ref{line:addLr1} or line \ref{line:addLr2}.
Since $h'$ is not pure w.r.t. $F$ then $h$ is not pure w.r.t. $F$.
The same reasoning shows that if $\pure_{t-1}(F') \ge 2$, then $g' \notin \Ex$.
Therefore,  
$g \cup g' $ is not added to $\Ex$
by line \ref{line:additionL-1}.
\end{enumerate}
\end{proof}

\begin{proposition}
Every family is created by Algorithm \ref{alg:Families2}
containing  at least
two clusters.
Moreover, if $F_C$ is created by case (b) of Algorithm \ref{alg:Families2}
at  iteration $t$, then at the beginning of this iteration
exactly two families  that belong to $C$
 have exactly two pure clusters and 
all the other families in $C$ have at most one pure cluster.
\label{prop:FC}
\end{proposition}
\begin{proof}
All the families created at line \ref{line:initial-families}  
have at least two pure clusters.
We use induction on the number of iterations.

At the beginning of iteration $1$ all connected
components of $G$ have only one family.
Let $g$ and $g'$ be the two clusters merged when $t=1$.
If a family $F_C$ is created (line \ref{line:FC-creation}) at this iteration,  then
the merging of $g$ and $g'$ must produce a connected component with two families.
Thus, we conclude that  
 $g$ is pure with respect to a family $F$ and 
$g'$ is pure w.r.t. to a family $F'$, with $F' \ne F$.
Assume w.l.o.g. that $|F'| \ge |F|$.
If  (a) occurs we must have $|F|=2$ and $|F'|>2$ and
if (b) occurs we must have $|F|=2$ and $|F'|=2$.
Furthermore, If (a) occurs $F_C$ will have at least two clusters, those
in $F' \setminus \{g'\}$. If (b) occurs $F_C$ will also have at least two clusters, 
$g \cup g'$ and the pure cluster in $(F \cup F') \setminus \{g,g'\}$
that is not added to $\Ex$ by line \ref{line:addLr2}.  
Thus, the result holds at the first iteration.

Let $t>1$. We assume by induction that 
the result holds for iteration $t-1$.
We analyze iteration $t$.
We split the proof into two cases:

Case 1) $F_C$ is created due to case (a) at iteration $t$. 

The definition of the case (a) assures that  there 
is a family $F$ in the connected component $C$ with $\pure_t(F)\ge 2$.
Thus, the
 pure clusters w.r.t. $F$ are added to $F_C$, so
 that $F_C$ is created with at least two clusters.

\medskip

Case 2) $F_C$ is created due to case (b) at iteration $t$. 

The definition of case (b) assures that  $|C| \ge 2$ and all families in $C$ have at most one pure cluster
after the merge of iteration $t$.
Moreover, Proposition \ref{prop:families-evol} assures that at most
two families have their number of pure clusters decreased when two clusters
are merged. Thus,  at least $|C|-2$ families that lie in $C$
 have 0 or 1  pure cluster at the beginning of iteration $t$,
otherwise, case (b) cannot occur.

We argue that 
  exactly two families that lie in $C$ have at least 2 pure clusters at the beginning
  of iteration $t$.
For the sake of a contradiction, assume that either 0 or 1 family that lies in $C$ has
 at least 2 pure clusters at the beginning of iteration $t$.
We consider two scenarios:
\begin{itemize}
\item The component $C$ does exist at the beginning of $t$,
that is, it was not produced by the union of two components at iteration $t$.
In this scenario, $C$ would satisfy either case (a) or case (b) at iteration $t-1$,
and $C$ would be removed from $G$.
Thus, this scenario cannot occur.

\item The component $C$ does not exist at the beginning of $t$.
In this scenario,  $C$ is the union of two connected components
of $G$ at the beginning of iteration $t$, that is, $C =  C' \cup C''$.
Let us assume $|C'| \ge |C''|$. If $|C'| \ge 2$,
then $|C'|$ would satisfy  either case (a) or case (b) at iteration $t-1$
and, as a consequence, would be removed from  $G$; this is not possible.
If $|C'|=|C''|=1$, 
then by our assumption one of these components, say $C'$,
has a family with at most one pure cluster at the beginning of $t$.
Then 
by induction, $C'$ does not correspond to a family created
at iteration $t-1$.
Thus, $C'$  satisfies 
 case (c) at iteration $t-1$,
so that $C'$  would be removed from $G$.
Thus, this scenario cannot occur as well.
\end{itemize}

Let $H$ and $H'$ be the families in $C$ that have at least 2 pure clusters
at the beginning of iteration $t$.
If one of them, say $H$, has more than 2 pure clusters, then $C$ does not satisfy case (b) because we would have $\pure_t(H) \ge 2$.
Thus, $\pure_{t-1}(H)=\pure_{t-1}(H') =2$ and
$\pure_{t-1}(H'')<2$ for every family $H''$ in $C \setminus \{H,H'\}$.

It remains to show that $F_C$ is created with at least two clusters.
Since (b) occurs, we must have that $g$ is pure w.r.t $H$ and $g'$ is pure
w.r.t. $H'$. Moreover, item 4 of Proposition \ref{prop:families-evol}
guarantees that $g \cup g' \notin \Ex$. Thus, $g \cup g'$ is added to $F_C$.
Moreover, either $H$ or $H'$ will have a pure cluster that is not
added to $\Ex$ by line \ref{line:addLr2} and, thus, this cluster will be added to $F_C$.
  \end{proof}

\begin{proposition}
Let $C$ be the connected component
associated with family $F_C$.
If $F_C$ is created by case (a) of Algorithm \ref{alg:Families2},
then $|C|-1$ families in $C$ have a pure cluster
added to $\Ex$ by line \ref{line:addLr1}.  
If $F_C$ is created by case (b) of Algorithm \ref{alg:Families2},
then $|C|-2$ families in $C$ have a pure cluster
added to $\Ex$ by  line \ref{line:addLr1} and  one family
has  a pure cluster
added to $\Ex$ by  line \ref{line:addLr2}.
In both cases, $C$ has exactly one family that adds no cluster to $\Ex$. 
\label{prop:LS-addition}
\end{proposition}
\begin{proof}

By Proposition \ref{prop:FC}, every family is created with at least two pure clusters.
By Proposition \ref{prop:families-evol},
at each iteration
the number of pure clusters in a family
either remains the same or decreases by one unit.

If $F_C$ is created due to case (a), then exactly $|C|-1$
families in $C$  have at most one pure cluster.
Thus, all of them reach line \ref{line:addLr1} and the result holds.

If $F_C$ is created due to the case (b), then 
Proposition \ref{prop:FC} guarantees
that at the beginning of the iteration in which $F_C$ is created, $C$ has
 exactly two families, say $H$ and $H'$, with
exactly two pure clusters and
all others
with at most one pure cluster.
Therefore, 
every family in $C- \{H,H'\}$  reaches line \ref{line:addLr1} and,
thus, add a cluster to $\Ex$.
Furthermore, exactly one family in $\{H,H'\}$ adds a cluster to $\Ex$ in line
\ref{line:addLr2}.
\end{proof}

\begin{proposition} 
The total  number of clusters added to $\Ex$ by lines \ref{line:addLr1} and \ref{line:addLr2} 
of  Algorithm \ref{alg:Families2} is at most  $k$.
\label{prop:L-Bound}
\end{proposition}
\begin{proof}
Let $m_{>1}=|\{T^*_i| T^*_i \mbox{ has at least 2 points}\}|$ and
$m_{1}=|\{T^*_i| T^*_i  \mbox{ has exactly 1 point} \}|$.
Initially,    $m_1$ clusters are added to $\Ex$ (line \ref{line:test}).

Let $D_1,D_2, \ldots, D_p$ be the trees of the forest  $D$
at the end of the Algorithm \ref{alg:Families2}.
Moreover, let $int(D_i)$ and $leaves(D_i)$ be, respectively, the
set of internal nodes and leaves of $D_i$.
Fix an internal node $v$ in $D_i$.
Note that $v$ corresponds to a family $F_C$ for some connected
component $C$ of $G$ that has at least two families (line \ref{line:parent}). Each child of $v$ corresponds to a family in $C$ and 
 by Proposition \ref{prop:LS-addition} all of them but one add a cluster to $\Ex$.
 Hence,  we can associate
to  $v $, $|children(v)|-1$ clusters
that are added to $\Ex$. 
Thus, the number of 
 clusters added to $\Ex$ is at most 
\begin{align*}
p+\sum_{i=1}^p \left( \sum_{v \in int(D_i)} (|children(v)|-1) \right)= 
p+\sum_{i=1}^p \left( - |int(D_i)| + \sum_{v \in int(D_i)} |children(v)| \right) =\\
p+\sum_{i=1}^p \left(- |int(D_i)|+|D_i|-1 \right) =
\sum_{i=1}^p leaves(D_i)= |leaves(D)|=m_{>1},
\end{align*}
where the term $p$  is due to the roots of the trees $D_1,\ldots,D_p$
since they can also add pure clusters to $\Ex$.

Therefore, the total number of clusters in $\Ex$ never
exceeds $m_{>1}+m_{1}=k$.
\end{proof}

The next proposition characterizes the clusters
created by \complink.
 
\begin{proposition} 
At the beginning of iteration $t$ of Algorithm \ref{alg:Families2},
 each cluster $h \in \C^{t-1}$ satisfies exactly one of the following
possibilities:
\begin{itemize}
\item[(i)] $h 
\in   
\Ex$;
\item[(ii)] $h \notin \Ex$ and $h$ is pure  w.r.t a family in $G$;
\item[(iii)] $h \notin \Ex$, $h$ is non-pure 
and there is a component $C$ in $G$
such that $h \subseteq  \bigcup_{H \in C} \Pts(H)$. 
\end{itemize}
\label{prop:clusters-structure}
\end{proposition}
\begin{proof}
At the beginning of Algorithm \ref{alg:Families2},  each cluster
$h$ satisfies (i) or (ii).
We assume by induction that this property
holds at the beginning of iteration $t$
and prove that it also holds 
at the beginning of iteration $t+1$.

We first argue that right after the merge of iteration $t$
every cluster in $\C^{t}$ satisfies one of the desired conditions  and,
then, we argue that these clusters still satisfy the desired  conditions
by the end of iteration $t$ or, equivalently, at the beginning of iteration
$t+1$.

Let $h$ be a cluster in $\C^{t}$.
If $h$ also belongs to 
to $\C^{t-1}$ then, by induction,
$h$ satisfies one of the conditions at the beginning of iteration $t$.
If $h$ satisfies  (i)  (resp.  (ii)), then
it also satisfies (i) (resp. (ii))  right after the merge because
the assumption that $h \in \C^t$ guarantees that $h$ is not merged at iteration
$t$. If $h$ satisfies (iii), then $h \subseteq
\bigcup_{H \in C} \Pts(H)$, for some connected
component of $G$. Hence,
$h$  satisfies the condition 
$h \subseteq \bigcup_{H \in C^{new}} \Pts(H)$,
after the merge, where $C^{new}$ is the component in $G$ that contains
the families of $C$ after the merge. 

If $h$ does not belong to   $\C^{t-1}$,
then  $h = g \cup g'$, where $g$ and $g'$ are the clusters
merged at iteration $t$. 
If $g \in \Ex$ or $g' \in \Ex$ then $h \in \Ex$, 
so it satisfies (i) after the merge.
If neither $g \in \Ex$ nor $g' \in \Ex$ then 
$h$ is not added to  $\Ex$. Moreover, 
if both $g$ and $g'$ are pure with respect to
the same family $H$, then $h$ is also pure w.r.t. $H$, so it satisfies
(ii).  Otherwise, $h$ is not pure and it satisfies (iii) when it is created.

We have just proved that every cluster in $\C^{t}$ satisfies
 one of the conditions of the proposition right after the merge.
 Now we show that each cluster 
 in $\C^{t}$ satisfies
 one of the conditions at the end of the iteration $t$.

Let $h \in \C^t$. We have the following cases:
\begin{itemize}
    \item $h$ satisfies (i) after the merge. Then,
it will satisfy (i) at the beginning of iteration $t+1$.

\item $h$ satisfies (ii) after the merge. 
Then,  $h$ is pure w.r.t. some family $H$. 
Let $C$ be the component where the family $H$ lies. 
If $C$  meets the conditions of cases (a) or  (b) then
$h$  will satisfy (ii)  at the beginning of iteration $t+1$.
because $h$ will be pure with respect to the new family $F_C$.
If $C$ meets the conditions of case (c) then $h$ is added
to $\Ex$, so it satisfies (i) at the beginning of iteration $t+1$.
If $C$ does not meet any of the cases, then
$h$  will satisfy either (i) or  (ii)  at the beginning of iteration $t+1$.

\item $h$ satisfies (iii) after the merge. 
Let $C$ be a component of $G$, right after the merge,  such that $h \subseteq
\bigcup_{H \in C} \Pts(H)$.
Note the $|C| \ge 2$ and, thus $C$ cannot meet case (c).
If $C$  meets either  case (a) or  (b) then
$h$  will satisfy (ii)  at the beginning of iteration $t+1$
because $h$ will be pure with respect to the new family $F_C$.
If  $C$  does not  meet (a) or (b), then
 $h$  will satisfy (iii)  at the beginning of iteration $t+1$.
\end{itemize}
\end{proof}

Now we
state and prove two propositions that, together, 
directly imply the correctness of Lemma  \ref{lem:2pureclusters}.

\begin{proposition}
Every connected component $C$ in $G$ satisfies one of the following conditions: 
(i) $|C|=1$ and the only family of $C$ has at least
two pure clusters  or (ii) $|C|>1$ and there exist  two  families in $C$
such that
each of them has at least two pure clusters.
\label{prop:2pureclusters-p1}
\end{proposition}
\begin{proof}
For the sake of reaching  a contradiction, let 
$t$ be the first iteration for which there is a component 
 $C$ in $G$ that does not satisfy the conditions
of the lemma at the beginning of iteration $t$.
Note that $t \ge 2$.

If $|C|=1$ and the only family $F$ in $C$ does not have
at least 2 pure clusters, then $C$ cannot be  the component associated with the family $F_C$ 
created by either case (a) or (b) at
iteration $t-1$ because Proposition \ref{prop:FC}
guarantees that the component where $F_C$ lies satisfies condition (i).
Thus, $C$  satisfies the condition of case (c) at iteration
$t-1$ and then it  is removed from 
$G$ at iteration $t-1$, which contradicts its existence at the beginning of iteration $t$

If both $|C| \ge 2$ and $C$ has only one family with
at least 2 pure clusters, then $C$
would satisfy case (a) at iteration $t-1$ and it would be removed
from $G$,  which contradicts its existence at the beginning of iteration $t$.
Similarly, 
if $|C| \ge 2$ and it has no family with
at least 2 pure clusters then $C$
would satisfy case (b) at iteration $t-1$ and it would be removed
from $G$,  which again contradicts its existence at the beginning of iteration $t$

We have established the proposition.
\end{proof}

\begin{proposition}
If all the nodes/families of $G$ are
removed at iteration $t'$ then $t'=n-k$
\label{prop:2pureclusters-p2}
\end{proposition}
\begin{proof}
We first argue that at the beginning of iteration $t'$,
$G$ has exactly one connected component.  
For the sake of reaching a contradiction, let us assume that $G$ has  two components
say $C$ and $C'$. 
By Proposition \ref{prop:2pureclusters-p1},  $C$   has one family,
say $F$, with $\pure_{t'-1}(F) \ge 2$.
Let $g$ and $g'$ be the clusters merged at iteration $t'$. 
If neither $g$ not $g'$ is pure with respect to $F$
we will have  $\pure_{t'}(F) \ge 2$ and the component
where $F$ lies after the merge will not satisfy the condition of 
case (c). Hence, there still be nodes in   $G$ by the end of $t'$.
Thus, one of the clusters, say $g$, is pure w.r.t. $F$. 
By using the same reasoning we conclude that 
$C'$   has at least one family,
say $F'$, with $\pure_{t'-1}(F') \ge 2$
and $g'$ is pure w.r.t. $F'$.
Since $g \cup g' \notin \Ex$ (item 4 of Proposition \ref{prop:families-evol})
then $g \notin \Ex$ and $g' \notin \Ex$, so that 
the merge of $g$ and $g'$ will create a component $C \cup C'$ in $G$
that has at least two families; this component does not 
not satisfy the conditions of 
case (c), so $G$ will have nodes by the end of iteration $t'$

We have proved that if $G$ 
has all its nodes removed at iteration $t'$, then
there is only one component in $G$ at the beginning of 
iteration $t'$. Let $C$ be this component: $C$ must have exactly one family,
say $F$, and  
$F$ must have at most two pure clusters, otherwise
case (c) is not reached and 
$G$ will still have nodes by the end of iteration $t'$.

Since $|C|=1$ no cluster satisfies 
condition (iii) of Proposition \ref{prop:clusters-structure}
and, thus, the total of clusters at the beginning of iteration 
$t'$ is given by the  number of clusters
 that satisfy either condition (i) or (ii)
 of Proposition \ref{prop:clusters-structure}.
 
Proposition \ref{prop:L-Bound} guarantees
that the number of clusters $h \in \Ex$ is at most $k$.
Thus, if $F$ has less than two pure clusters at the beginning of $t'$,
then the number of clusters $h$ that satisfies
(i) or (ii) of Proposition \ref{prop:clusters-structure}.
is at most $k+1$.

 If $F$ has two pure clusters at the beginning of $t'$
and $C$ satisfies case (c) at iteration $t'$, then
a pure cluster in $F$ is added to $\Ex$.
Thus,  Proposition \ref{prop:L-Bound} guarantees
that $\Ex$ has at most $k-1$ clusters at the beginning of iteration  $t'$.
Hence,  the number of clusters that 
that satisfies  (i) or (ii) of Proposition \ref{prop:clusters-structure}
is, again, at most $(k-1)+2=k+1$. 

Thus, the total  number at the beginning of iteration
$t'$ is at most $k+1$ and, hence,
  $t' = n-k$.
\end{proof}

\noindent {\em Proof of Lemma  \ref{lem:2pureclusters}.}
It follows directly from Propositions 
\ref{prop:2pureclusters-p1} and \ref{prop:2pureclusters-p2}.

\section{Useful inequalities}

\begin{proposition} Let $p=\log_2 3  -1$ and let $a,b,x,y$ real numbers with 
 $0 \le a,b$ and $x,y \ge 1$.
Moreover, let $a x^p \ge b y^p$.
 Then,
 $$  a x^p + 2 by^p \le (a+b)(x+y)^p$$
\label{prop:calculations-avg}
\end{proposition}
\begin{proof}
Let $$f(a,b,x,y)=(a+b)(x+y)^p -
 ax^p - 2 by^p .$$ We have that
 $$ \frac{\partial f}{\partial a} =(x+y)^p- x^p >0.$$

Since $a x^p \ge b y^p$, in the minimum of $f$,
we must have $a = b (y/x)^p$.
When $a= b (y/x)^p$,
we have that
$$f(a,b,x,y)=b \left ( \left( \frac{y}{x} \right) ^p+1)(x+y)^p -
 3 by^p \right ) = \\ b \left ( (y^p+x^p)(x+y)^p - 3 x^py^p \right)$$
Since $b \ge 0$ it is enough to prove  that
$(y^p+x^p)(x+y)^p - 3 x^py^p \ge 0$
 
By the AGM inequality
$$ x^p+y^p \ge 2 (x^{p/2})(y^{p/2})$$
and
$$ (x+y)^p \ge ( 2(xy)^{1/2})^p $$

By multiplying these inequalities we get

$$( x^p+y^p)(x+y)^p  \ge 2^p 2 (x^py^p)=3x^py^p$$
 \end{proof}

\begin{proposition}  The following holds
$$ \frac{\log 6}{\log 4} = \max \left \{ \frac{\log (2i-2)}{  \log i} | i \mbox{ is an integer larger than } 1 \right \} $$  
\label{prop:calculo-alpha}
\end{proposition}
\begin{proof}
We can inspect manually that
$ \log(2i-2) / \log i \le \frac{\log 6}{\log 4},$ for 
every integer smaller than 11.
For $i \ge 11$ we have
 that
$$  \frac{\log(2i-2)}{\log i} < \frac{\log (2i)}{\log i} = 
1+ \frac{1} {\log i} \le 1+ \frac{ 1} {\log 11}
\le \frac{\log 6}{\log 4} $$.
\end{proof}

\begin{proposition} 
Let $p$ be a real number that satisfies $p \ge \log_i (2i-2),$ for every $i>1$.
Moreover, let $\ell$ be a positive number larger than 1 
and let  $1 \le a_1 \le a_2 \ldots \le a_\ell $.
Then, $$ a_\ell^p+ a_{\ell-1}^p + \sum_{i=1}^{\ell-2} 2a_i^p  \le
\left ( \sum_{i=1}^\ell a_i\right )^p$$
\label{prop:calculations2}
\end{proposition}
\begin{proof}
Let $\avec=(a_1,\ldots,a_\ell)$.
We define
$$f(\avec):= \left( \sum_{i=1}^{\ell} a_i \right )^p -
\left( a_{\ell}^p+ a_{\ell-1}^p + \sum_{i=1}^{\ell-2} 2a_i^p \right).$$
We need to show that $f(\avec) \ge 0$, for every valid $\avec$.

First, we consider the case $\ell=2$.
In this case, 
$$ \frac{\partial f}{\partial a_2}= p(a_1+a_2)^{p-1} -  p  a_2^{p-1} >0 $$
Thus, in the minimum of $f$ we must have $a_1=a_2$.
When this happens,
$$f(\avec)=(2a_1)^{p}-2a_1^p  > 0.$$  

Now, we consider the case $\ell>2$.
For $t \le \ell-2$, we
define 
$$f^t(\avec):=\left( \sum_{i=1}^{t-1} a_i + (\ell-t+1)a_t \right )^p -
\left( 2(\ell - t) a_t^p+  \sum_{i=1}^{t-1} 2a_i^p \right)$$
Note that $f^t(\avec)=f(\avec')$
where $a'_i=a_i$ for $i <t$ and $a'_i=a_t$ for $i \ge	 t$. 
We show that the minimum of $f$ is equal to the minimum of $f^t$.


For $j > \ell-2$, 
$$ \frac{\partial f}{\partial a_j}=
p( \sum_{i=1}^\ell a_i)^{p-1} -  p a_j^{p-1} > 0.$$ 
Hence, in the minimum of $f$ we have that $a_{\ell-2}=a_{\ell-1}=a_{\ell}$.
Thus, the minimum of $f$ equals the minimum of $f^{\ell-2}$.
Now, we show that the minimum of $f^t$ is equal to the minimum of
$f^{t-1}$ for $t \le \ell-2$.

We have that
\begin{align*}
\frac{\partial f^t}{\partial a_t} = (\ell-t+1) p\left ( \sum_{i=1}^{t-1} a_i + (\ell-t+1)a_t   \right )^{p-1} - 2(\ell - t) p a_t^{p-1} > \\ 
 p (\ell-t+1)^p a_t^{p-1} - 2(\ell - t) p a_t^{p-1} \ge \\
0,
\end{align*}
where the second inequality holds because
the definition of $p$ assures that 
$p \ge \frac{ \log 2 (\ell - t) }{  \log (\ell-t+1)}.$
Hence, in the minimum of $f^t$ we have that $a_t=a_{t-1}$, so that the minimum of $f^t$ equal the minimum of $f^{t-1}$.
Therefore, we can conclude that the minimum of $f$ is equal to the minimum 
of $f^1$.

Now we note that 
$$f^1(\avec) = (\ell a_1)^p-(2\ell-2)(a_1)^p
\ge (2\ell-2)(a_1)^p - (2\ell-2)(a_1)^p=0,$$
where the inequality holds because $p \ge \log_\ell (2\ell-2)$
\end{proof}

\end{document}